\documentclass{article}

\usepackage[final,main]{arxiv_own_style}

\usepackage{graphicx}
\usepackage{wrapfig}
\usepackage{placeins}
\usepackage[english]{babel}
\usepackage{subcaption}
\usepackage{siunitx}
\usepackage{booktabs}
\usepackage{pifont}
\usepackage{floatflt}
\usepackage{tabularx}
\usepackage{listings}
\usepackage{xcolor}
\usepackage{amsmath}
\usepackage{amssymb} 

\usepackage{amsmath,amsfonts,bm}

\def\eqref#1{equation~\ref{#1}}

\def\1{\bm{1}}

\DeclareMathAlphabet{\mathsfit}{\encodingdefault}{\sfdefault}{m}{sl}
\SetMathAlphabet{\mathsfit}{bold}{\encodingdefault}{\sfdefault}{bx}{n}

\DeclareMathOperator{\Tr}{Tr}

\usepackage{natbib}
\usepackage{hyperref}
\usepackage{url}

\definecolor{codegreen}{rgb}{0,0.6,0}
\definecolor{codegray}{rgb}{0.5,0.5,0.5}
\definecolor{codepurple}{rgb}{0.58,0,0.82}
\definecolor{backcolour}{rgb}{0.95,0.95,0.92}

\lstdefinestyle{mystyle}{
  backgroundcolor=\color{backcolour}, commentstyle=\color{codegreen},
  keywordstyle=\color{magenta},
  numberstyle=\tiny\color{codegray},
  stringstyle=\color{codepurple},
  basicstyle=\ttfamily\footnotesize,
  breakatwhitespace=false,         
  breaklines=true,                 
  captionpos=b,                    
  keepspaces=true,                 
  numbers=left,                    
  numbersep=5pt,                  
  showspaces=false,                
  showstringspaces=false,
  showtabs=false,                  
  tabsize=2
}
\usepackage[utf8]{inputenc} 
\usepackage[T1]{fontenc}    
\usepackage{hyperref}       
\usepackage{url}            
\usepackage{booktabs}       
\usepackage{amsfonts}       
\usepackage{nicefrac}       
\usepackage{microtype}      
\usepackage{xcolor}         
\usepackage{amsmath}
\usepackage{amsthm}
\usepackage{amssymb}
\usepackage{todonotes}
\usepackage{multirow}
\usepackage{soul}

\newtheorem{theorem}{Theorem}
\newtheorem{lemma}{Lemma}

\title{Diffeomorphic Optimization}

\author{
  Ludwig Winkler\thanks{Work done during internship at Genentech}\\
  Microsoft Research \\ 
  AI4Science \\
  \And
  Andrew Leaver Fay \\
  Prescient Design, Genentech Roche \\
  \And
  Joseph Kleinhenz \\
   Prescient Design, Genentech Roche  \\
  \And
  Pan Kessel \\
   Prescient Design, Genentech Roche \\
}

\begin{document}

\maketitle

\begin{abstract}

Generative models learn data distributions that reside on a low-dimensional manifold within a higher-dimensional ambient space. 
Optimizing differentiable objectives on this manifold is challenging: the ambient loss landscape is high-dimensional, rugged, and non-convex. 
Direct gradient descent, blind to the manifold's geometry, quickly drifts off it.
Diffeomorphic optimization starts from the observation that diffusion and flow models provide a map from the data manifold to a much simpler base space in which we perform gradient descent.
Using differential geometry, we show this is equivalent to Riemannian gradient descent on the data manifold up to $\mathcal{O}(\lambda^2)$ corrections, keeping trajectories on-manifold by construction and yielding a smoother optimization surface. For protein design, we extend diffeomorphic optimization to the matrix Lie groups $\mathrm{SO}(3)$ and $\mathrm{SE}(3)$, deriving an autograd-compatible $\mathrm{SO}(3)$ gradient and a generalized adjoint-state method for backpropagation through Lie-group ODE solvers. Diffeomorphic optimization improves over tuned guidance on secondary-structure targeting with FrameFlow ($91.3\%$ vs. $63.3\%$ of residues in the Ramachandran target), outperforms OC-Flow on peptide binding affinity at $2\times$ the speed, and reduces Rosetta energies by thousands of units across the PDB test set for structures with hundreds of residues.

\end{abstract}

\section{Introduction}
Machine learning data is typically concentrated on a low-dimensional manifold \citep{manifold1, manifold2, manifold3}.
In practice, these manifolds are not known explicitly. Many objectives we care about, such as the stability of a protein or the  fidelity of a generated image, are in fact defined on low-dimensional data manifolds embedded in high-dimensional spaces.
Direct optimization in this space is difficult: the objective landscape is highly non-convex, the data manifold is implicit, and unconstrained updates, with e.g. direct gradient descent in the data space, tend to produce out-of-distribution solutions.

Diffusion and flow-based generative models provide a means to address this difficulty. These models learn a diffeomorphic map from a tractable base distribution to the data distribution.
This transformation reparameterizes the data manifold: complex, multimodal regions in data space correspond by design  to smoother and more regular regions in the base space of the prior distribution. 
The geometry of the manifolds required for optimization are already implicitly learned by these models.
While generative models have been highly successful in sampling from such manifolds, their use as a foundation for subsequent optimization is less developed.

\begin{figure}[h]
  \centering
  \includegraphics[width=0.8\linewidth]{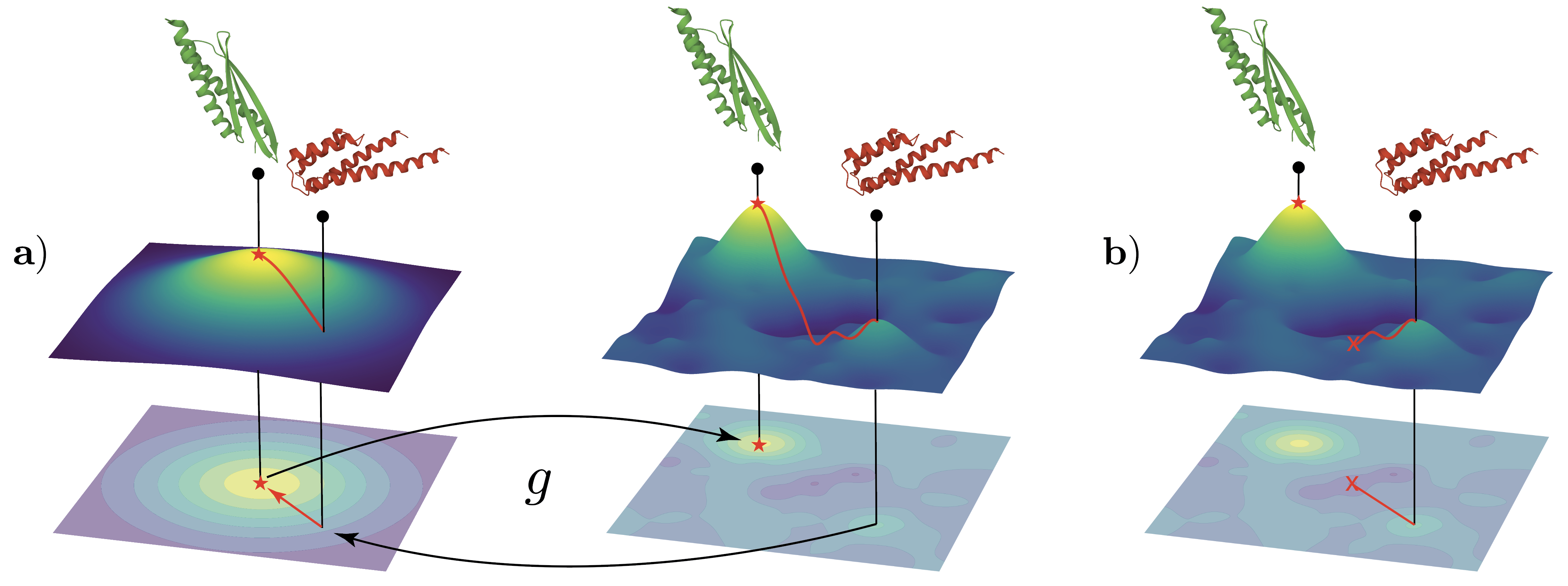}
  \caption{Optimization of protein with undesired properties (red) to a protein of desired properties (green). \textbf{Right:} Direct optimization in target space gets stuck in local minima. \textbf{Left:} Optimization in base space leads to smoother objective that allows for mode-switching.}
  \label{fig:aufmacher}
\end{figure}


In this work, we propose diffeomorphic optimization - a novel method to optimize arbitrary differentiable cost functions on the data manifold. 
Our method harnesses the recent advances in flow matching and diffusion models, is applicable to a wide range of tasks, and scales to high dimensional data, such as proteins with several hundreds of residues.
The proposed framework yields a inference time compute method for frozen, pretrained diffusion and flow models.
For a visual intuition of this idea, we refer to Figure~\ref{fig:aufmacher}.

Existing methods that try to optimize objectives with pretrained models come with caveats: guidance approximates an intractable conditional density \cite{dhariwal2021diffusion}; reward fine tuning requires retraining \cite{clark2023draft}, and optimal control methods inject control into the flow and can pull trajectories off manifold \cite{wang2024training}.
None of these apply to diffeomorphic optimization.

The approach diffeomorphic optimization takes is the following.
Diffusion and flow models learn a diffeomorphic (smooth and invertible) map $g:\mathcal{Z} \to \mathcal{X}$ by integrating the ordinary differential equation $\smash{g(z) \equiv  x_1 = z + \int_0^1 d\tau \, v_\theta(x_\tau, \tau)}$ that maps a base space $\mathcal{Z}$, equipped with a simple probability density $q_{\mathcal{Z}}$, to the target space $\mathcal{X}$, for example, the space of atom positions. 

Consider a cost function $\mathcal{L}: \mathcal{X} \to \mathbb{R}$, e.g. the energy or thermostability of a protein, that we want to optimize on the data manifold $\mathcal{D}$. Since the map $g$ is diffeomorphic, we change coordinates to the base space variables,
\begin{align}
    \mathcal{L}\circ g: \mathcal{Z} \to \mathbb{R}, \;\; z \mapsto \mathcal{L}(g(z)) \,.
\end{align}
We can then perform gradient descent with learning rate $\lambda \in \mathbb{R}$ in the base space variables
\begin{align}
    z^{(i+1)} = z^{(i)} - \lambda \nabla_z \mathcal{L}(g(z^{(i)}))\,,  && \textrm{for} \; i \in \{1, \dots, n\}\,,
\end{align}
and map the final base point $z^{(n)}$ to its corresponding $x^{(n)} = g(z^{(n)})$.
This parameterization has several advantages: the map $g$ is bijective and thus no information is lost. The data distribution is often highly multimodal. However, the density of the base space $q_{\mathcal{Z}}$ is chosen to be a simple unimodal distribution, such as a normal distribution $\mathcal{N}(0,1)$. Thus, the base space reparameterization makes it easier to switch modes leading to a smoother loss landscape. 

The optimization task is complicated by the fact that generative models for proteins often do not sample the atom positions directly but rather the $SE(3)$ backbone frames and $SO(2)$ sidechain dihedral angles \cite{framediff, frameflow, alphaflow, alphafold2}. We will therefore derive two efficient methods to facilitate backpropagation through ODE solvers on matrix Lie groups in Section~\ref{sec:solvers_on_se3}. The first relies on repurposing existing autograd engines for calculating the Riemannian gradient. For the second, we derive a suitable generalization of the adjoint-state method for matrix Lie groups.

We showcase diffeomorphic optimization in several numerical experiments. Specifically, we demonstrate that we can start from a given protein and then controllably change (parts of) its secondary structures using the flow matching model FrameFlow \cite{framediff} as the diffeomorphic map. Using the DiffDock diffusion model \cite{diffdock}, we show that we can optimize the Vina docking score \cite{vina} for a protein-ligand pocket and peptide, and improve the stability and affinity of peptides \cite{wang2024training}.
Furthermore, we demonstrate that we can minimize the Rosetta energy function of a given protein in the base space of the AlphaFlow model \cite{alphaflow}.
This relaxation protocol leads to significantly lower energies than the state-of-the-art Rosetta Relax.

Our work coincides with a recent rise of inference computation for protein design. Practitioners often sample thousands of designs, rank, and submit only a handful for further wetlab experiments \cite{rfab, alphafold3, rfdiffusion, frey2025lab}. 
The reason for this is that experimental wet lab capacity (and not sampling costs) is the main bottleneck. Furthermore, extensive sampling has been shown to lead to pronounced performance improvements, for example, for antibody-antigen coupling \cite{alphafold3}. 
Diffeomorphic optimization provides a more targeted approach to obtain high quality samples compared to brute-force sampling followed by ranking and selection. 
It does not require finetuning (unlike RL or reward finetuning), auxiliary networks (as in guidance) or rejection sampling (best-of-N).
It can be applied to any frozen, pretrained, flow or diffusion model with a differentiable reward.
We therefore believe that it will become an important part of the inference compute toolbox for proteins.

Briefly summarized, our key contributions are the following.
We propose diffeomorphic optimization that is applicable to any differentiable optimization objective. As we show theoretically, it automatically enforces the manifold constraint and leads to a smoother optimization landscape. 
To enable this, we propose two effective methods to facilitate backpropagation through matrix Lie group ODE solvers and will provide efficient PyTorch implementations for it.
Experimentally, we apply diffeomorphic optimization to tasks of high practical relevance, i.e. protein ligand docking, Rosetta energy relaxation, and secondary structure modification, by combining it with state-of-the-art generative models in the field of protein generation, specifically FrameFlow, DiffDock, and AlphaFlow.

\section{Related Works}

Protein hallucination is a version of computational protein design which uses backpropagation through a folding model to its input sequence \cite{anishchenko2021novo,kosugi2022solubility,goverde2023novo,pacesa2024bindcraft, cho2025boltzdesign1}.
Gradient descent and normalizing flows has been explored in the explainability literature to generate counterfactual explanations \cite{joshi2019towards,dombrowski2021diffeomorphic, dombrowski2023diffeomorphic,dhurandhar2018explanations} although not for flow matching and diffusion models. 
\citet{dflow} proposes to differentiate through flows for controlled generation. Our work builds on this reference by generalizing it to matrix Lie groups, which is of high relevance for proteins. We also provide a detailed theoretical analysis of the method and derive an efficient adjoint state method. 
\citet{wang2024training} explores related ideas for matrix groups in the framework of optimal control. Specifically, the authors propose a nice framework which adds an additive control to the vector field of the flow. Our approach does not use a control but rather optimizes the initial condition following \cite{dflow}. We discuss the relationship to this reference in more detail in Appendix~\ref{app:related_work} and compare in detailed numerical experiments to their approach. \citet{liu2023flowgrad} similarly relies on control variables but directly applies gradient descent to them. This reference does however not consider matrix groups.
The adjoint state method on manifolds has been discussed in other works for charts \cite{lou2020neural, mathieu2020riemannian} or particular manifolds \cite{rezende2020normalizing, winkler2024bridging, bacchio2023learning, albergo2022building}.
Guidance is a widely used method to bias diffusion and flow-matching models towards certain desiderata is guidance. There exist various flavors of it, such as classifier-based guidance \cite{dhariwal2021diffusion}, classifier-free guidance \cite{ho2022classifier}, and universal guidance \cite{bansal2023universal}. Further information can be found in the appendix~\ref{app:related_work}.

\section{Diffeomorphic Optimization stays on Manifold} 
\label{sec:restriction_to_manifold}

In this section, we analyze the diffeomorphic optimization procedure theoretically using differential geometry. In particular, we will demonstrate its relation with gradient descent on the manifold. 
\subsection{Basics Concepts of Differential Geometry}
\textbf{Manifolds and coordinates:} a manifold $\mathcal{M}$ is a space that locally takes the form of $\mathbb{R}^D$, similar to the earth which can locally approximated by flat three-dimensional space. More formally, for each point $p \in \mathcal{M}$, there exists a chart $\varphi: U \to \mathbb{R}^D$ where $U$ is an open subset of $\mathcal{M}$ containing the point $p$. In this sense, the manifold is locally described by a Euclidian $D$ space. The pair $(U,\varphi)$ is called coordinate chart and the component functions $x^i$ of $\varphi(p) = (x^1(p), \cdots , x^D(p))$ are called coordinates. 

\textbf{Tangent space:} the tangent space $T_p \mathcal{M}$ contains the velocity vectors $\frac{d}{dt} \gamma(t) |_{t=0}$ of curves $\gamma: \mathbb{R} \to \mathcal{M}$ with $\gamma(0)=p$. It can be shown that the tangent space $T_p\mathcal{M}$ is a $D$-dimensional vector space. Let $(U,\varphi)$ be a coordinate chart on $\mathcal{M}$ with coordinates $x$. We can then define $\varphi \circ \lambda_k(t) = (x^1(p),\dots,x^k(p) + t,\dots,x^D(p))$ with $k \in \{1,\dots,D\}$. This implicitly defines curves $\lambda_k : \mathbb{R} \to \mathcal{M}$ through $p$. We denote the corresponding tangent vectors as $\tfrac{\partial}{\partial x^k} =\tfrac{d}{dt}  \lambda_k(t)|_{t=0}$ and it can be shown that they form a basis of the tangent space $T_p\mathcal{M}$. The differential $dg_p: T_p\mathcal{M} \to T_{f(p)}\mathcal{N}$ for a map $g: \mathcal{M} \to \mathcal{V}$ between two manifolds $\mathcal{M}$ and $\mathcal{V}$ is a linear map between the respective tangent spaces. The curve $\gamma_v:\mathbb{R} \to \mathcal{M}$, corresponding to the vector $v \in T_p \mathcal{M}$, is mapped by the differential $dg_p$ to the curve $g \circ \gamma_v: \mathbb{R} \to \mathcal{N}$. Thus, the differential acts as $dg_p[v] = \frac{d}{dt}g \circ \gamma_v(t)|_{t=0}$.

\textbf{Riemannian metric:} A Riemannian manifold is endowed with an inner product $\langle \cdot, \cdot \rangle_{G_p}: T_p \mathcal{M} \times T_p \mathcal{M} \to \mathbb{R}$ for each tangent space, which allows us to define a notion of the length of tangent vectors. $G$ is also known as the Riemannian metric. We refer to the excellent textbook \cite{lee2018introduction}.

\textbf{Riemannian gradient:} The Riemannian gradient $\textrm{grad}^G_p f \in T_p\mathcal{M}$ at point $p\in \mathcal{M}$ of a function $f: \mathcal{M} \to \mathbb{R}$ is the vector uniquely defined by the relation
\begin{align}
    \forall v \in T_p \mathcal{M}: &&  \langle v, \textrm{grad}^G_p f \rangle_{G_p} = df_p[v] 
\end{align}
and crucially depends on Riemannian metric $G$ of the manifold. In coordinates, the Riemannian gradient is $\widehat{\textrm{grad}_p^G f} =\sum_j G^{ij} \partial_j f$ where $G^{ij}$ is the inverse of the metric tensor $G_{ij} = G(\tfrac{\partial}{\partial x^i},\tfrac{\partial}{\partial x^j})$.

\textbf{Exponential map:} On a Riemannian manifold, we can further define geodesics, i.e. curves $\gamma:I \subset \mathbb{R} \to \mathcal{M}$ that have vanishing acceleration with respect to the covariant derivative induced by the Riemannian connection. For each $p \in \mathcal{M}$ and $v \in T_p\mathcal{M}$, there exists a unique geodesic $\gamma_v$ that satisfies $\gamma_v(0) = p$ and $\frac{d}{dt} \gamma_v(t)|_{t=0} = v$.
The exponential map is then defined by
\begin{align}
    \exp_p: T_p &\mathcal{M} \to \mathcal{M}, \;\; \;\; v \mapsto \gamma_v(1) \,.
\end{align}
Intuitively, the exponential map formalizes the notion of taking a step in the direction of the tangent vector $v$ and then finding the closest point on the manifold.

\textbf{Gradient descent on manifolds:} adding and scaling points on the manifold will not necessarily lead to points on the manifold. In contrast, such operations acting on tangent vectors do lead to other tangent vectors, since a tangent space is a vector space. This suggests a generalization of gradient descent on manifolds: we first scale the gradient of a loss $\mathcal{L}:\mathcal{M} \to \mathbb{R}$, which in particular is a tangent vector, by the learning rate $\lambda \in \mathbb{R}$. To obtain the updated point $p^{i+1}\in \mathcal{M}$, we apply the exponential map at the previous point $p^{i}\in \mathcal{M}$ to the rescaled gradient:
\begin{align}
    p_{i+1} \leftarrow \exp_{p_i}(-\lambda \, \textrm{grad}^G_{p_i} \mathcal{L}) \,.
\end{align}
In $\mathbb{R}^n$, the exponential map is $\exp_p(v) = p + v$ leading to standard gradient descent.

\subsection{Diffeomorphic optimization is equivalent to Gradient Descent on the Data Manifold} 
Let $g: \mathcal{Z} \to \mathcal{D}$ be a generative model that maps its latent space $\mathcal{Z}$ to the data manifold $\mathcal{D}$. Given such a model, we can then perform gradient descent in its latent space and map the resulting point on the data manifold. We then show:

\begin{theorem}\label{th:on_manifold}
Let $g: \mathcal{Z} \to \mathcal{D}$ be a diffeomorphic generative model and $\mathcal{L}: \mathcal{D} \to \mathbb{R}$ is the loss. Then, up to quadratic corrections in the learning rate $\lambda$, performing gradient descent in the latent space $\mathcal{Z}$ and then mapping it go the data space $\mathcal{D}$ with $g$ is equivalent to gradient descent on the data manifold, i.e.,
\begin{align}
    g(\exp_z(- \lambda \, \textrm{grad}^{\tilde{G}}_z \mathcal{L} \circ g ) )= \exp_{g(z)} ( - \lambda \, \textrm{grad}^G_{g(z)} \mathcal{L} + \mathcal{O}(\lambda^2) ) \,,
 \end{align}
where $G(u,v)=\tilde{G}(dg^{-1}u, dg^{-1}v)$ denotes the pushforward of the Riemannian metric $\tilde{G}$ on $\mathcal{Z}$. 
\end{theorem}
\begin{proof}
    See Appendix~\ref{app:proof_on_manifold}.
\end{proof}
In coordinates, the gradient in data space is given by
\begin{align}
    \widehat{\textrm{grad}}_{g(\hat{z})} \mathcal{L} = J_g \, \widehat{\textrm{grad}}_z \mathcal{L} \,, && 
    \textrm{with} \; J_g = \frac{\partial \hat{g}}{\partial \hat{z}} \,.
\end{align}
We can perform a SVD decomposition of the Jacobian $J_g$. In the basis of its left singular values, the learning rate will be scaled by the corresponding singular vectors. The data manifold $\mathcal{D}$ is heavily concentrated and thus we expect low values for the singular vectors along these contracted directions. This gives us a mechanism to extract the implicitly learned tangent space from the generative diffusion or flow-matching model, which is (approximately) spanned by the remaining left-singular vectors.

\section{Backpropagation through $\mathbf{SE(3)}$ ODE solvers} \label{sec:solvers_on_se3}


The frame representation of proteins considers an idealized backbone geometry of its heavy atoms $[N, C_\alpha, C, O] \in \mathbb{R}^{3,4}$ determined by fixed bond lengths and angles \cite{alphafold2, framediff}. A protein consists of multiple residues and for each residue $i$, the main backbone atoms can be described by a simple rotation $R_i \in SO(3)$ and translation $t_i \in \mathbb{R}^3$ of these idealized backbone coordinates, as seen in Figure~\ref{fig:se3},
\begin{wrapfigure}{R}{0.3\textwidth}
\centering
\includegraphics[width=0.3\textwidth]{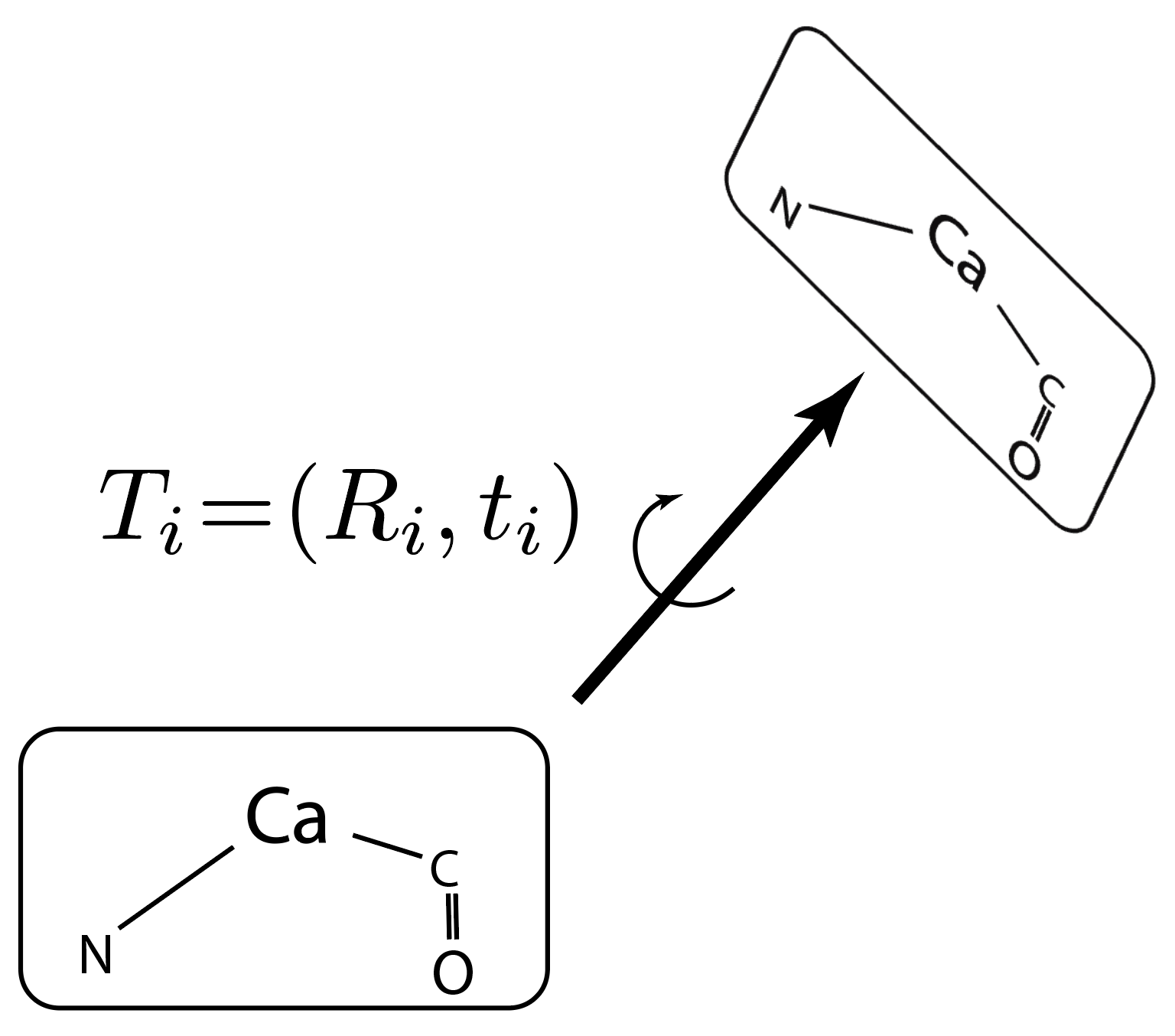}
\caption{\label{fig:se3} $SE(3)$ roto-translation of the idealized backbone positions.}
\end{wrapfigure}
\begin{align}
    [N_i, C_i,, (C_\alpha)_i] = T_i [N, C, C_\alpha]
\end{align}

where $T_i = (R_i, t_i)$ is an element of the three-dimensional special Euclidean group $SE(3) = SO(3) \ltimes \mathbb{R}^3$. The positions of the remaining heavy atoms (backbone oxygen and sidechain carbons) can be fixed by dihedral angles. Therefore, the vector field $v_\theta$ in the ODE takes value in the tangent space of the corresponding Lie groups $SE(3)$ and $SO(2)$. Integration on $SO(2)$ can simply be performed by standard numerical integration in $\mathbb{R}^3$ and then wrapping the result on the circle $S^1$. The case of $SE(3)$ is, however, highly non-trival. For example, the group $SE(3)$ has the product rule
\begin{align}
    T_1 T_2 = (R_1, t_1) (R_2, t_2) = (R_1 R_2, R_1 t_2 + t_1) \,.
\end{align}
The translation group $\mathbb{R}^3\simeq \{ (I,t) | t \in \mathbb{R}^3\}$, where $I$ is the unit element of $SO(3)$, is an abelian normal subgroup of $SE(3)$ since $(R,t) (I, t') (R,t)^{-1} = (I, R t')$. Therefore $SE(3)$ is not semi-simple and there is no canonical left- and right-invariant Riemannian metric induced by the Killing form. In the generative model literature, one typically chooses the Riemannian metric to be the sum of the metrics of $SO(3)$ and $\mathbb{R}^3$ \cite{framediff}. This choice is right-invariant but not left-invariant\footnote{The question of whether the metric is left- or right-invariant is a matter of convention. We choose the Riemannian metric to be induced by push-forward of the Killing form with respect to the right multiplication.}. For this choice, the exponential maps and Riemannian gradients for the translations and rotations decouple
\begin{align}
    \exp_{T}((\omega, v) = (\exp_{R}(\omega), \exp_{t}(v) ) \,, && \textrm{grad}_T f = (\textrm{grad}_Rf, \textrm{grad}_tf) \,,
\end{align}
for $\omega \in \mathfrak{so}(3)$, $v \in \mathfrak{r}^3 \simeq \mathbb{R}^3$, and $f:SE(3) \to \mathbb{R}$. The explicit form of these exponential maps and gradients for both translations and rotations are:

\textbf{Translations:} since $\mathbb{R}^3$ is a vector space, it holds that
\begin{align}
    \exp_t(v) = t + v \,, && \textrm{grad}_t f = \nabla_t f(t, R)
\end{align}
where we use the canonical isomorphism $\mathfrak{r}^3 \simeq \mathbb{R}^3$ and $\nabla_t$ denotes the standard gradient operator $\nabla = (\partial_1, \partial_2, \partial_3)$. As a result, we recover standard gradient descent for the translational part
\begin{align}
    t^{i+1} = t^{(i)} - \lambda \nabla_{t^i} \mathcal{L}(t^i, R^i) \,.
\end{align}

\textbf{Rotations:} as we discuss in Appendix~\ref{app:sec:gradientdescentonliegroups}, the right multiplication $R_g: G \to G$, $h \mapsto hg$ induces a isomorphism between the tangent spaces $T_gG$ and the Lie algebra $\mathfrak{g} \simeq T_e G$ for any Lie group $G$. It is often easier to work with the exponential map $\exp: \mathfrak{g} \to G$ in terms of Lie algebra elements. For matrix Lie groups, such as $SO(3)$, this exponential map is simply given in terms of the matrix exponential $\exp(v) = \sum_{i=0}^\infty \frac{v^n}{n!}$ where $v$ is the matrix Lie algebra element. Similarly, we can uniquely associate the Riemannian gradient $\textrm{grad}_R f$ of a function $f:SO(3) \to \mathbb{R}$ with the Lie algebra element
\begin{align}
    \nabla f(R) = \sum_a T^a \nabla ^a f \in \mathfrak{so}(3) \,, && \nabla^a f(R) = \frac{d}{dt} f(\exp(t T^a) R) \bigg|_{t=0} \,, \label{eq:riemanniangrad}
\end{align}
where $T^a$ denote the antisymmetric generators of the Lie algebra $\mathfrak{so}(3)$. Therefore, gradient descent on $SO(3)$ amounts to
\begin{align}
    R^{i+1} = \exp(- \lambda \nabla_{R^i} \mathcal{L}(R^i, t^i)) R^i \,.
\end{align}

\textbf{Summary:} Gradient descent on $SE(3)$ can be performed by
\begin{align}
T^{i+1} = 
\begin{bmatrix}
    R^{i+1} \\ 
    t^{i+1} 
\end{bmatrix}= 
\exp_{T^i}(-\lambda \nabla_{T^i} \mathcal{L}(T^i))  =    
\begin{bmatrix}
     \exp\left(-\lambda \nabla_{R^i} \mathcal{L}(R^i, t^i)\right) R^i \\ 
    t^{(i)} - \lambda 
\nabla_{t^i} \mathcal{L}(t^i, R^i) 
\end{bmatrix} 
\end{align}
In this sense, the gradient descent of the rotational and translational part decouples.

\subsection{Gradients of $\mathbf{SE(3)}$ Solvers}\label{sec:adjoint_state_method_on_se3}
Let $g: SE(3)^n \to SE(3)^n$ be a diffeomorphism. We restrict to $n=1$ for notational simplicity but the results immediately generalize to $n>1$. We use the notation $(R, t) = g(Z, z)$ with $Z \in SO(3)$ and $z \in \mathbb{R}^3$ denoting the base space variables. We can thus reparameterize the loss function $\mathcal{L}: SE(3) \to \mathbb{R}$ by
\begin{align}
    \mathcal{L}(g(Z,z))
\end{align}
for which we want to perform gradient descent with respect to the base variables $(Z, z) \in SE(3)$. 
In particular, we can consider flow matching and diffusion models for which the diffeomorphism $g$ is defined by $g(Z,z) \equiv (Z_1, z_1)$ where the right hand side is the solution of the following initial value problem on $SE(3)$:
\begin{align}
    dT_\tau = \left(dZ_\tau, dz_\tau\right) = \left(\; V_\theta(Z_\tau, z_\tau), \,v_\theta(Z_\tau,z_\tau) \; \right) \,d\tau \,, \quad \text{with} \quad 
    T_0 = (Z_0, z_0) = (Z, z) \label{eq:forwardODE}\,,
\end{align}
where $\tau \in [0,1]$ and $V_\theta \in \mathfrak{so}(3)$, $v_\theta \in \mathbb{R}^3$ are parameterized by neural networks. 
As a result, we need to backpropagate through the numerical solver of the ODE. Specifically, we need to calculate the Riemannian gradient $\nabla_Z \mathcal{L}(g(Z,z))$ with respect to rotation $Z \in SO(3)$. In the following, we will propose two methods to do so. 

\textbf{Repurposing Autograd for $\mathbf{SO(3)}$:}
the Riemannian gradient can be expressed in terms of a simple matrix derivative:
\begin{theorem}\label{th:autograd}
    The Riemannian gradient \eqref{eq:riemanniangrad} on $SO(3)$ of a loss function $\mathcal{L}:SO(3) \to \mathbb{R}$ can be written as
    \begin{align}
    \nabla \mathcal{L}(R) =  2\left[\frac{df}{dR} R^\top \right]_A\,,\label{eq:autogradtrick}
\end{align}
where $\frac{df}{dR}$ denotes the standard matrix-calculus gradient with respect to the matrix $R \in SO(3)$. Furthermore, we denote the antisymmetric component of a matrix $M$ by $[M]_A = \frac12(M - M^\top)$. 
\end{theorem}
\begin{proof}
    See Appendix~\ref{app:proof_of_autograd}.
\end{proof}

This suggests a straightforward way to calculate Riemannian gradients with modern autograd frameworks. Specifically, one wraps any $SO(3)$-valued leaf variable in an autograd method acting as the identity in the forward pass but performing the multiplication \eqref{eq:autogradtrick} in its backward where $\smash{\frac{df}{dR}}$ is the output gradient. The gradient of this variable will correspond to the antisymmetric Riemannian gradient matrix \eqref{eq:riemanniangrad},{
see Appendix~\ref{app:riemannian_gradient_code} for a very concise pytorch implementation of this. We emphasize that this approach is completely general and ensures that the Riemannian gradient is seamlessly integrated in existing autograd functionality.

In particular, we can facilitate backpropagation through ODE solvers on $SO(3)$ by combining this repurposing trick with standard activation checkpointing at intermediate points of the integration trajectory. This allows us to limit the memory footprint of the computational graph to the desired degree.

\begin{figure}[h]
  \centering
 \includegraphics[width=0.7\linewidth]{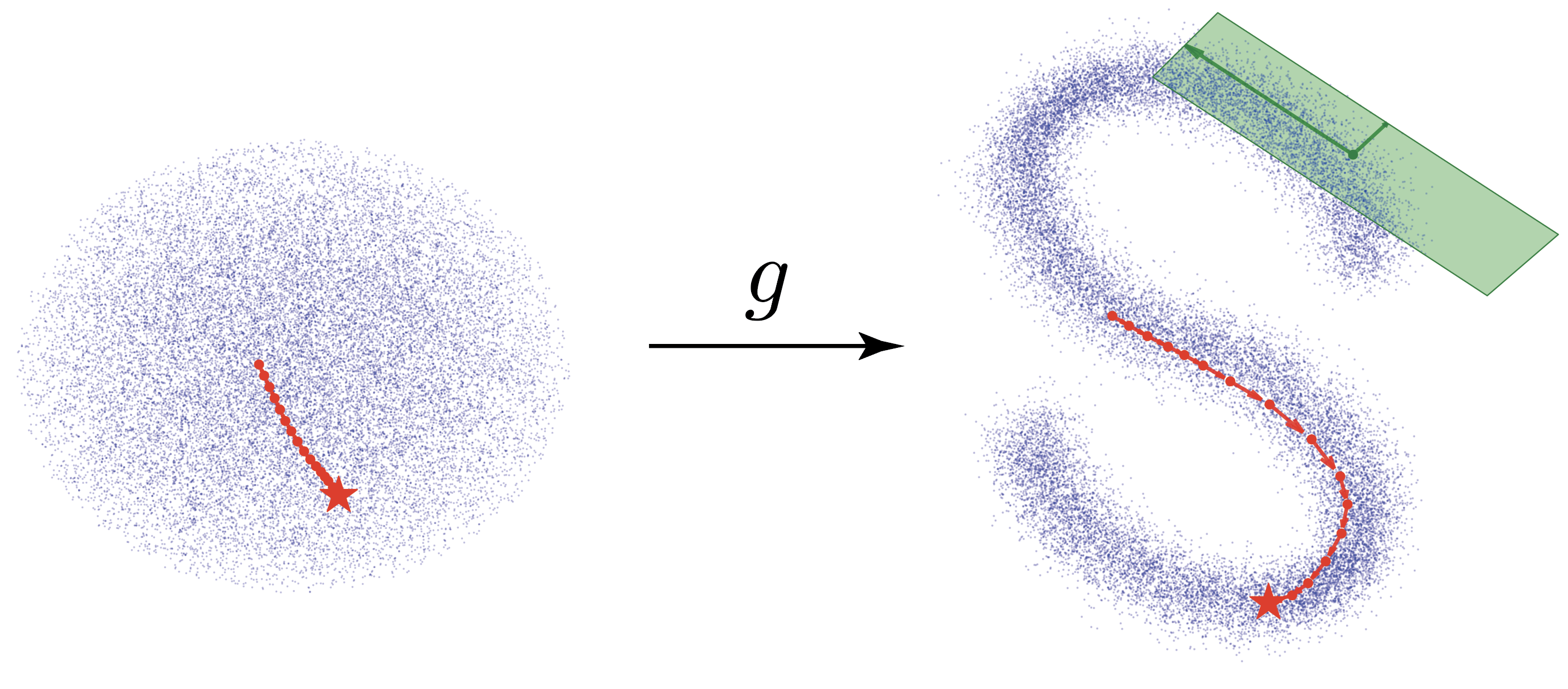}
  \caption{Diffeomorphic Optimization on $SO(3)$: left hand side visualizes the gradient descent trajectory in the base space $\mathcal{Z}$. Right hand side visualizes the same trajectory when mapped to the target space $\mathcal{X}$. Diffeomorphic optimization clearly stays on manifold. The green plane is spanned by the left-singular vectors of the Jacobian $\nabla_Z g(Z)$ scaled by the corresponding singular vectors. This shows that the diffeomorphic map indeed captures the tangent space of the data manifold to good approximation.}
  \label{fig:toy_example_so3}
\end{figure}

\textbf{Adjoint State Method on $\mathbf{SE(3)}$:}
the repurposing method has the advantage that it is rather general. However, for flow matching and diffusion models, the gradients can be written in terms of a particular adjoint state ODE. For the translational part, this is well known \cite{neuralode} and given by
\begin{align}
    \frac{d\mathcal{L}(g(Z,z))}{dz} = a_0
\end{align}
where the adjoint state $a_0 \in \mathbb{R}^3$ is obtained by solving the following terminal value problem
\begin{align}
    \frac{da_\tau}{d\tau} = -\sum_{i=1}^3 a^i_\tau \frac{dv^i_\theta(Z_\tau, z_\tau)}{dz_\tau} \,, && a_1= \frac{d\mathcal{L}(R,t)}{dt} \,,
\end{align}
with $\tau \in [0,1]$. In the appendix, we derive a generalization of the adjoint state method for $SO(3)$:
\begin{theorem}\label{th:adjoint_state}
    The Riemannian gradient  of the reparameterized loss function 
    \begin{align}
    \nabla_Z \mathcal{L}(g(Z,z)) = A_0
    \end{align}
    can be obtained by the following terminal value problem for the Lie-algebra-valued adjoint state $A_\tau=\sum_{i=1}^3 A^i_\tau T^i \in \mathfrak{so}(3)$:
    \begin{align}
        \frac{d A_\tau}{d\tau} &= [V_\theta(Z_\tau, t_\tau), A_\tau] -  \sum_{i=1}^3 A^i_\tau \nabla_{Z_\tau} V_\theta^i(Z_\tau, t_\tau) \,, &&
        A_1 = \nabla_R \mathcal{L}(R, t) \,, \label{eq:backwardadjstateso3} 
    \end{align}
    where $\tau \in [0,1]$ and $[A,B]=AB-BA$ is the commutator of matrices $A$ and $B$.
\end{theorem}
\begin{proof}
    See Appendix~\ref{app:proof_of_adjoint_state}.
\end{proof}

\cite{wang2024training} used the adjoint method on $SO(3)$ to calculate gradients with respect to deterministic optimal control which they used as guidance.
As shown in the appendix, this result also holds for any matrix Lie group. A notable advantage of the adjoint state method is that it turns the gradient calculation into a terminal value problem and thus can benefit from modern ODE solvers. However, the adjoint state method has the downside that it assumes that the forward pass has no discretization errors. Depending on the stiffness of the ODE, one may therefore accumulate a systematic error by the discretization effects of both the forward ODE \eqref{eq:forwardODE} and the backward adjoint state ODE \eqref{eq:backwardadjstateso3}. In practice, we recommend relying on numerical experiments to benchmark the relative performance of the adjoint state method or the repurposed autograd method for the problem at hand. We will provide a modular integration library that implements both methods.

\begin{figure}[h]
  \centering
\includegraphics[width=1.0\linewidth]{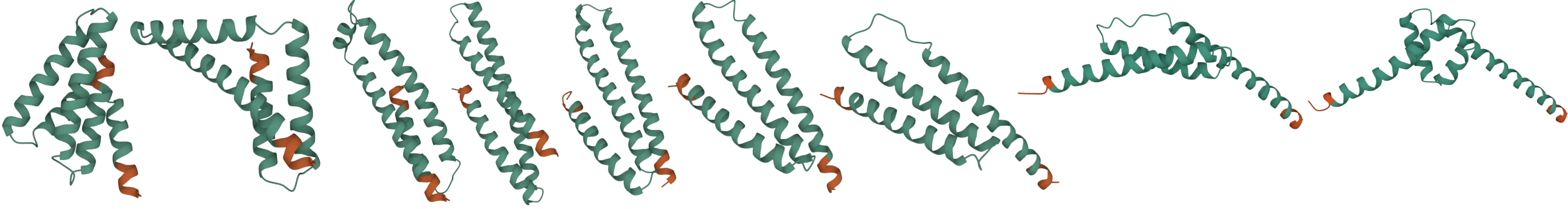}
  \caption{We show various snapshots of the trajectory of diffeomorphic maximization of the distance between the termini of the proteins (marked in red). The optimization traverses a number of plausible conformations.}
  \label{fig:ends}
\end{figure}

\section{Experiments}
We demonstrate the generality of diffeomorphic optimization in several numerical experiments. We refer to Appendix~\ref{app:experiments} for more details.

\textbf{$\mathbf{SO(3)}$ Manifold:}
In Figure~\ref{fig:toy_example_so3}, we visualize diffeomorphic optimization for a toy dataset on $SO(3)$. Specifically, we create a dataset of S-shape in the angle-axis vectors of the rotations. We then train a flow matching model to sample $SO(3)$ elements on this manifold. The diffeomorphic optimization aims to push the sample as close as possible to the point marked by a star on the S-shaped data manifold. 
As the figure demonstrates, diffeomorphic optimization leads to a gradient descent trajectory on the data manifold. 
This illustrates that diffeomorphic optimization can harness the differential geometric notions learned by the flow model.

\begin{figure}[h]
  \centering
    \includegraphics[width=0.8\linewidth]{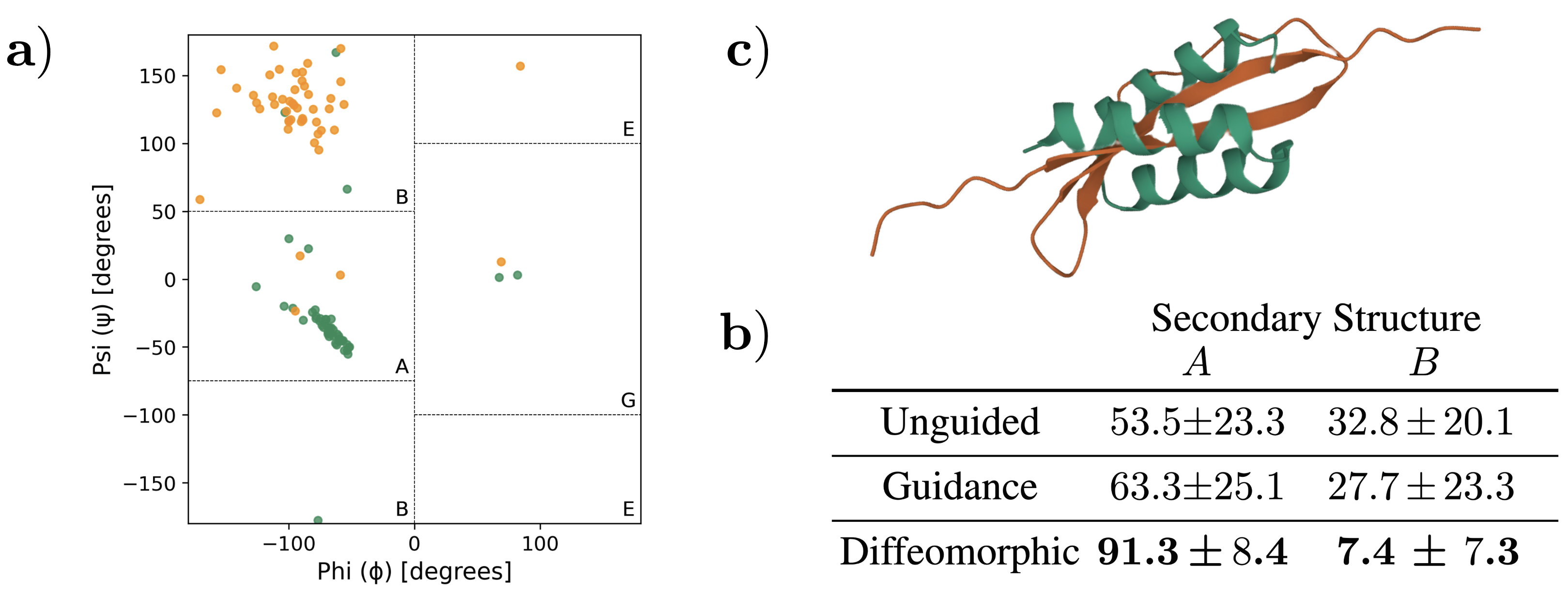}
  \caption{Diffeomorphic optimization is used to change the secondary structure of the protein. For this, we use the ABEGO classification shown on the left. We start from the green protein shown on the top right whose residues are mostly clustered in the A region of the Ramachandran plot which corresponds to $\alpha$-helices. We then optimize the structure such that these residues are pushed to the B region corresponding to $\beta$-sheets. The resulting protein structure is shown in orange in the top right. The final and initial Ramachandran plot are shown on the left. The lower right shows comparison to guidance for 50 samples obtained from the model. Diffeomorphic optimization significantly outperforms guidance despite careful tuning of the baseline hyperparameters described in the appendix.}
  \label{fig:rama}
\end{figure}

\begin{table}[h!]
\centering
\caption{Evaluation of OC-Flow peptide design.}
\resizebox{\textwidth}{!}{
\begin{tabular}{lccccccc}
\hline
 & MadraX $\downarrow$ & RMSD $\downarrow$ & SSR \% $\uparrow$ & BSR \% $\uparrow$ & Stability $\downarrow$ & Affinity $\downarrow$ & Diversity $\uparrow$ \\
\hline
Ground-truth        & -0.588 & --    & --    & --    & -84.893 & -36.063 & --    \\
PepFlow             & -0.195 & 1.645 & 0.794 & 0.874 & -45.660 & -26.538 & 0.310 \\
OC-Flow(trans)      & -0.229 & 1.774 & 0.797 & 0.876 & -48.380 & -27.328 & 0.323 \\
OC-Flow(rot)        & -0.221 & 1.643 & 0.794 & 0.872 & -48.636 & -27.211 & 0.310 \\
OC-Flow(trans+rot)  & -0.263 & 2.127 & \textbf{0.797} & 0.869 & -48.853 & -27.468 & 0.338 \\
\hline
DiffeoOpt       & \textbf{-0.309} & \textbf{1.605} & 0.796 & \textbf{0.881} & \textbf{-49.417} & \textbf{-28.409} & \textbf{0.340} \\
\hline
\label{table:OCFlowComparison}
\end{tabular}
}
\end{table}

\textbf{Peptide Design:}
we compare diffeomorphic optimization to the OC-Flow framework \cite{wang2024training} in which a control parameter $\theta_t$ is added to the flow in the form of $\smash{g(z) = z + \int_0^1 v_\theta(x_\tau) + \theta_\tau d\tau}$, see Appendix~\ref{app:related_work} for more details.
We kept the experimental setup identical to \citet{wang2024training} and compare to our proposed diffeomorphic optimization of both translation and rotations of the peptide's backbone. We use the repurposing approach and summarize the results in Table~\ref{table:OCFlowComparison}. We stress that this improvement can be obtained while being $2\times$ faster in terms of runtime. 

\textbf{Secondary Structure Modification with FrameFlow:} we use FrameFlow \cite{frameflow}, a widely used backbone generation flow matching model, to optimize the backbone structure with respect to a given cost function. Figure~\ref{fig:ends} illustrates an optimization trajectory for which the distance between the protein termini is maximized. 
This demonstrates that our method can optimize a given protein structure while staying on the data manifold. Similarly, Figure~\ref{fig:rama} shows that we can also optimize the secondary structure of a protein as specified by the ABEGO classification scheme \cite{abego, kim2009sampling}.

\begin{wrapfigure}{R}{0.35\textwidth}
\centering
\includegraphics[width=0.3\textwidth]{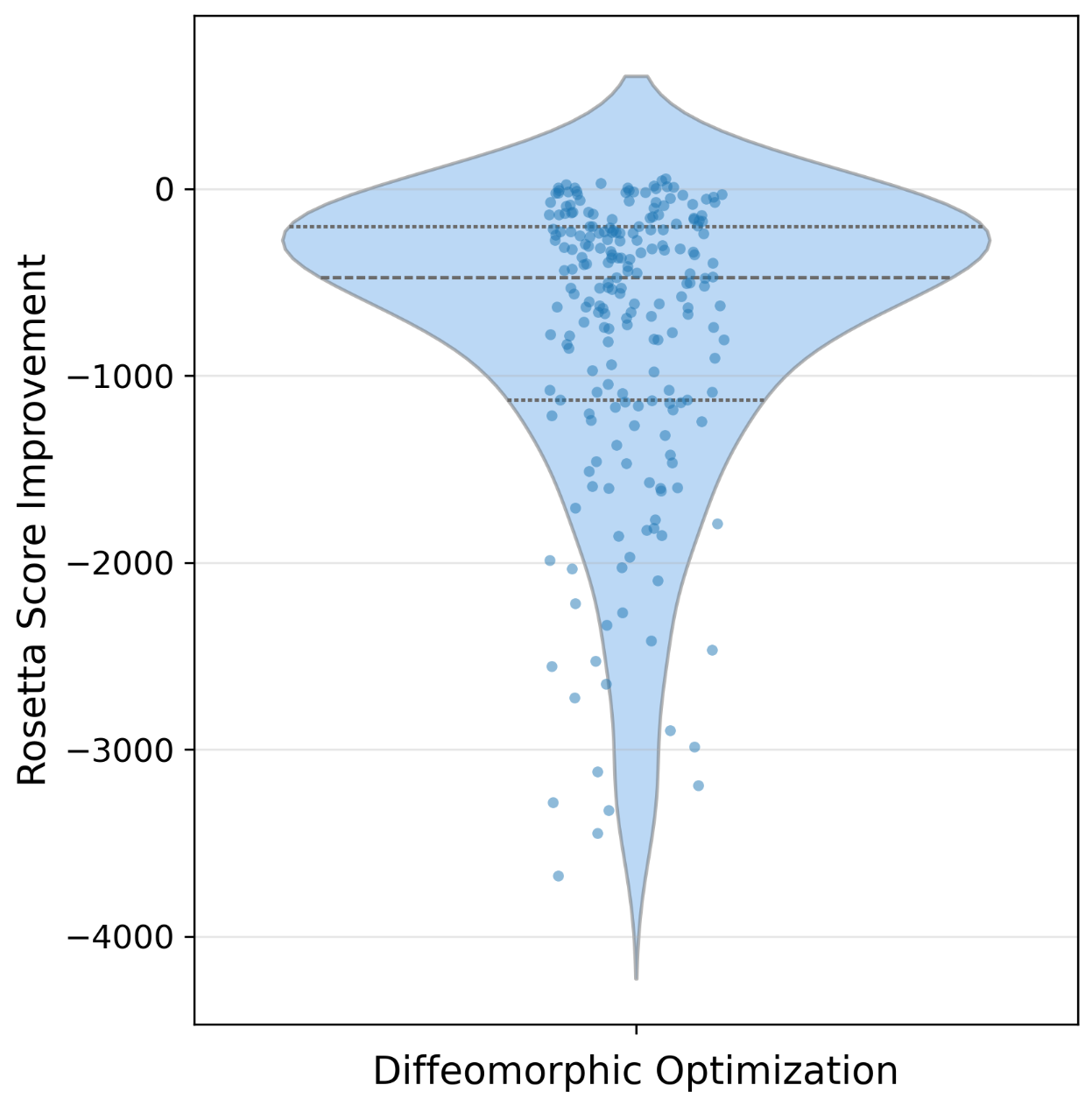}
\caption{\label{fig:rosetta_single}Diffeomorphic optimization applied to AlphaFlow exhibits strong improvement as measured in the Rosetta energy score.}
\end{wrapfigure}
For this, we minimize an energy function that is based on the Ramachandran-PAAPP-combined term in the Rosetta energy \cite{rosetta, tmol} which interpolates energies derived from statistics of backbone dihedral preferences in the PDB on a toroidal grid. The grid values are modified such that the undesired regions are disfavored. We compare to guidance with carefully tuned hyperparameters. 
Diffeomorphic optimization significantly outperforms the  while maintaining designability with results in Appendix~\ref{app:experiments_secondary}.

\textbf{Protein-Ligand Docking with DiffDock:} DiffDock is a diffusion model that samples both the center of mass translation and rotation (i.e. an $SE(3)$ element) and the torsion angles of the ligand for a given protein \cite{diffdock}. The model is trained on the PDBBind dataset \cite{wang2005pdbbind}, which contains experimentally determined protein-ligand docking structures. We then apply diffeomorphic optimization to maximize the popular VinaSF score of the protein-ligand complex, utilizing its OpenDock implementation \cite{opendock}.

To enable backpropagation through the solver, we modify the DiffDock generation process to use probability ODE sampling through our custom integration library. We then optimize the score function on pdb test set samples.
We compare diffeomorphic optimization to iid sampling followed by selecting the sample with the best docking score using the same computational budget. For this, each gradient descent step is counted as three sampling trajectories, which is quite generous for the baseline. Figure~\ref{fig:img2} shows the improvement $S_{\textrm{diffeo}}-S_{\textrm{iid}}$ obtained by diffeomorphic optimization which significantly outperforms this baseline. Our method can thus efficiently combine physics-based dockingr scores with trained generative models.
Further analysis is provided in Appendix~\ref{app:experiments_docking}.

\begin{wrapfigure}{R}{0.35\textwidth}
\centering
\vspace{-20pt}
\includegraphics[width=0.3\textwidth]{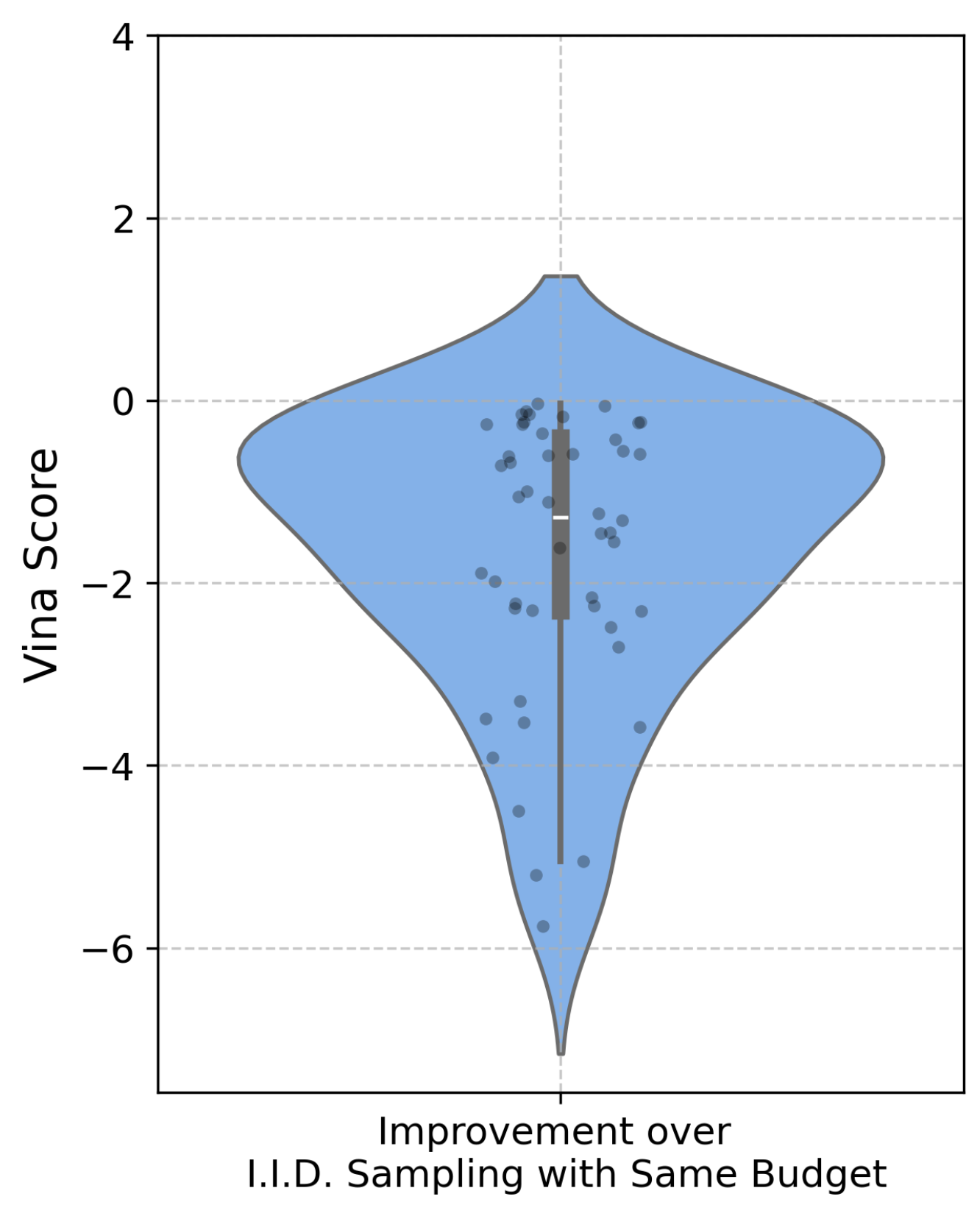}
\caption{\label{fig:img2}Diffeomorphic optimization leads to better/lower vina score than iid sampling with the same budget.}
\end{wrapfigure}

\textbf{Minimization of Rosetta Energy with AlphaFlow:} we consider the Rosetta energy function \cite{rosetta} which is widely used in the protein community using the tmol pytorch implementation \cite{tmol} of $\text{beta\_nov2016\_cart}$. We select the same pdb test set as in the AlphaFlow publication \cite{alphaflow}. 

This reference proposed several flows that are obtained from protein structure prediction models. 
We select the ESMFold model for convenience, as it does not require MSA processing.  We compare diffeomorphic optimization to the state-of-the-art Rosetta Relax protocol. 
From our theoretical analysis, we expect that diffeomorphic optimization will lead to better mode mixing. Once an energetically favorable mode is reached, it is cheaper to use standard relaxation to find its minimum. We thus first minimize the Rosetta energy function in the base space of the flow which is then followed by standard Rosetta Relax. 

Figure~\ref{fig:rosetta_single} shows that diffeomorphic optimization leads to substantially lower energy values throughout the pdb test set. Furthermore, the energy values of the baseline Rosetta Relax do not improve significantly by running a longer optimization or using more seeds. Thus, the relative advantage is not a function of the higher cost of diffeomorphic optimization.
Further analysis is provided in Appendix~\ref{app:experiments_tertiary}.

\section{Conclusion and Limitations} We have proposed diffeomorphic optimization which allows us to minimize arbitrary differentiable cost functions on the data manifold. 
This is achieved with off-the-shelf autograd engines with simple gradient wrappers making it a readily available tool for practitioners.
A downside of this approach are the considerable numerical costs due to the backpropagation through the generation. However, sampling costs are of lower concern in protein design where the main bottleneck is experimental wet-lab verification of the designs. 
It is completely standard in this setting to sample thousands of designs, rank them, and submit only a handful for wetlab experiments. Diffeomorphic optimization provides a more targeted approach to obtain high quality samples. 

\newpage
\bibliography{refs}

@article{manifold1,
  title={Testing the manifold hypothesis},
  author={Fefferman, Charles and Mitter, Sanjoy and Narayanan, Hariharan},
  journal={Journal of the American Mathematical Society},
  volume={29},
  number={4},
  pages={983--1049},
  year={2016}
}

@article{manifold2,
  title={Verifying the union of manifolds hypothesis for image data},
  author={Brown, Bradley CA and Caterini, Anthony L and Ross, Brendan Leigh and Cresswell, Jesse C and Loaiza-Ganem, Gabriel},
  journal={arXiv preprint arXiv:2207.02862},
  year={2022}
}

@article{manifold3,
  title={Hardness of Learning Neural Networks Under the Manifold Hypothesis},
  author={Kiani, Bobak and Wang, Jason and Weber, Melanie},
  journal={Advances in Neural Information Processing Systems},
  volume={37},
  pages={5661--5696},
  year={2024}
}

@article{frameflow,
  title={Fast protein backbone generation with se (3) flow matching},
  author={Yim, Jason and Campbell, Andrew and Foong, Andrew YK and Gastegger, Michael and Jim{\'e}nez-Luna, Jos{\'e} and Lewis, Sarah and Satorras, Victor Garcia and Veeling, Bastiaan S and Barzilay, Regina and Jaakkola, Tommi and others},
  journal={arXiv preprint arXiv:2310.05297},
  year={2023}
}

@article{framediff,
  title={SE (3) diffusion model with application to protein backbone generation},
  author={Yim, Jason and Trippe, Brian L and De Bortoli, Valentin and Mathieu, Emile and Doucet, Arnaud and Barzilay, Regina and Jaakkola, Tommi},
  journal={arXiv preprint arXiv:2302.02277},
  year={2023}
}

@inproceedings{alphaflow,
  title={AlphaFold Meets Flow Matching for Generating Protein Ensembles},
  author={Jing, Bowen and Berger, Bonnie and Jaakkola, Tommi},
  year={2024},
  booktitle={Forty-first International Conference on Machine Learning}
}

@inproceedings{diffdock,
    title={DiffDock: Diffusion Steps, Twists, and Turns for Molecular Docking}, 
    author = {Corso, Gabriele and Stärk, Hannes and Jing, Bowen and Barzilay, Regina and Jaakkola, Tommi},
    booktitle={International Conference on Learning Representations (ICLR)},
    year={2023}
}

@article{vina,
  title={AutoDock Vina: Improving the speed and accuracy of docking with a new scoring function, efficient optimization, and multithreading},
  author={Trott, Oleg and Olson, Arthur J},
  journal={Journal of Computational Chemistry},
  volume={31},
  number={2},
  pages={455--461},
  year={2010},
  publisher={Wiley Online Library},
  doi={10.1002/jcc.21334}
}

@article{albergo2022building,
  title={Building Normalizing Flows with Stochastic Interpolants},
  author={Albergo, Michael S. and Vanden-Eijnden, Eric},
  journal={arXiv preprint arXiv:2209.15571},
  year={2022},
  url={https://arxiv.org/abs/2209.15571}
}

@article{alphafold2,
  title={Highly accurate protein structure prediction with AlphaFold},
  author={Jumper, John and Evans, Richard and Pritzel, Alexander and Green, Tim and Figurnov, Michael and Ronneberger, Olaf and Tunyasuvunakool, Kathryn and Bates, Russ and {\v{Z}}{\'\i}dek, Augustin and Potapenko, Anna and others},
  journal={nature},
  volume={596},
  number={7873},
  pages={583--589},
  year={2021},
  publisher={Nature Publishing Group}
}

@misc{tmol,
    author       = {Leaver-Fay, Andrew and Flatten, Jeff  and Ford, Alex  and Kleinhenz, Joseph and Solberg, Henry  amd Baker, David and Watkins, Andrew M and Kuhlman, Brian and DiMaio, Frank},
    title        = {tmol: a GPU-accelarated, PyTorch implementation of Rosetta’s relax protocol (manuscript in preparation)},
    year         = {2025},
    url          = {https://github.com/uw-ipd/tmol}
  }

@book{lee2018introduction,
  title={Introduction to Riemannian manifolds},
  author={Lee, John M},
  volume={2},
  year={2018},
  publisher={Springer}
}

@article{neuralode,
  title={Neural ordinary differential equations},
  author={Chen, Ricky TQ and Rubanova, Yulia and Bettencourt, Jesse and Duvenaud, David K},
  journal={Advances in neural information processing systems},
  volume={31},
  year={2018}
}

@article{rosetta,
  title={The Rosetta all-atom energy function for macromolecular modeling and design},
  author={Alford, Rebecca F and Leaver-Fay, Andrew and Jeliazkov, Jeliazko R and O’Meara, Matthew J and DiMaio, Frank P and Park, Hahnbeom and Shapovalov, Maxim V and Renfrew, P Douglas and Mulligan, Vikram K and Kappel, Kalli and others},
  journal={Journal of chemical theory and computation},
  volume={13},
  number={6},
  pages={3031--3048},
  year={2017},
  publisher={ACS Publications}
}

@article{abego,
  title={Automatic classification and analysis of $\alpha$$\alpha$-turn motifs in proteins},
  author={Wintjens, Ren{\'e} T and Rooman, Marianne J and Wodak, Shoshana J},
  journal={Journal of molecular biology},
  volume={255},
  number={1},
  pages={235--253},
  year={1996},
  publisher={Elsevier}
}

@article{wang2005pdbbind,
  title={The PDBbind database: methodologies and updates},
  author={Wang, Renxiao and Fang, Xueliang and Lu, Yipin and Yang, Chao-Yie and Wang, Shaomeng},
  journal={Journal of medicinal chemistry},
  volume={48},
  number={12},
  pages={4111--4119},
  year={2005},
  publisher={ACS Publications}
}

@article{opendock,
    author = {Hu, Qiuyue and Wang, Zechen and Meng, Jintao and Li, Weifeng and Guo, Jingjing and Mu, Yuguang and Wang, Sheng and Liangzhen, Zheng and Wei, Yanjie},
    title = "{OpenDock: A pytorch-based open-source framework for protein-ligand docking and modelling}",
    journal = {Bioinformatics},
    pages = {btae628},
    year = {2024},
    month = {10},
    issn = {1367-4811},
    doi = {10.1093/bioinformatics/btae628},
    url = {https://doi.org/10.1093/bioinformatics/btae628},
    eprint = {https://academic.oup.com/bioinformatics/advance-article-pdf/doi/10.1093/bioinformatics/btae628/59933649/btae628.pdf},
}

@article{lou2020neural,
  title={Neural manifold ordinary differential equations},
  author={Lou, Aaron and Lim, Derek and Katsman, Isay and Huang, Leo and Jiang, Qingxuan and Lim, Ser Nam and De Sa, Christopher M},
  journal={Advances in Neural Information Processing Systems},
  volume={33},
  pages={17548--17558},
  year={2020}
}

@article{mathieu2020riemannian,
  title={Riemannian continuous normalizing flows},
  author={Mathieu, Emile and Nickel, Maximilian},
  journal={Advances in Neural Information Processing Systems},
  volume={33},
  pages={2503--2515},
  year={2020}
}

@article{bacchio2023learning,
  title={Learning trivializing gradient flows for lattice gauge theories},
  author={Bacchio, Simone and Kessel, Pan and Schaefer, Stefan and Vaitl, Lorenz},
  journal={Physical Review D},
  volume={107},
  number={5},
  pages={L051504},
  year={2023},
  publisher={APS}
}

@inproceedings{rezende2020normalizing,
  title={Normalizing Flows on Tori and Spheres},
  author={Rezende, Danilo Jimenez and Mohamed, Shakir},
  booktitle={International Conference on Machine Learning},
  pages={8083--8092},
  year={2020}
}

@article{dflow,
  title={D-flow: Differentiating through flows for controlled generation},
  author={Ben-Hamu, Heli and Puny, Omri and Gat, Itai and Karrer, Brian and Singer, Uriel and Lipman, Yaron},
  journal={arXiv preprint arXiv:2402.14017},
  year={2024}
}

@inproceedings{dombrowski2021diffeomorphic,
  title={Diffeomorphic explanations with normalizing flows},
  author={Dombrowski, Ann-Kathrin and Gerken, Jan E and Kessel, Pan},
  booktitle={ICML Workshop on Invertible Neural Networks, Normalizing Flows, and Explicit Likelihood Models},
  year={2021}
}

@article{dombrowski2023diffeomorphic,
  title={Diffeomorphic counterfactuals with generative models},
  author={Dombrowski, Ann-Kathrin and Gerken, Jan E and M{\"u}ller, Klaus-Robert and Kessel, Pan},
  journal={IEEE Transactions on Pattern Analysis and Machine Intelligence},
  volume={46},
  number={5},
  pages={3257--3274},
  year={2023},
  publisher={IEEE}
}

@article{alphafold3,
  title={Accurate structure prediction of biomolecular interactions with AlphaFold 3},
  author={Abramson, Josh and Adler, Jonas and Dunger, Jack and Evans, Richard and Green, Tim and Pritzel, Alexander and Ronneberger, Olaf and Willmore, Lindsay and Ballard, Andrew J and Bambrick, Joshua and others},
  journal={Nature},
  volume={630},
  number={8016},
  pages={493--500},
  year={2024},
  publisher={Nature Publishing Group UK London}
}

@article{kim2009sampling,
  title={Sampling bottlenecks in de novo protein structure prediction},
  author={Kim, David E and Blum, Ben and Bradley, Philip and Baker, David},
  journal={Journal of molecular biology},
  volume={393},
  number={1},
  pages={249--260},
  year={2009},
  publisher={Elsevier}
}

@article{rfab,
  title={Atomically accurate de novo design of single-domain antibodies},
  author={Bennett, Nathaniel R and Watson, Joseph L and Ragotte, Robert J and Borst, Andrew J and See, D{\'e}jena{\'e} L and Weidle, Connor and Biswas, Riti and Shrock, Ellen L and Leung, Philip JY and Huang, Buwei and others},
  journal={biorxiv},
  year={2024}
}

@article{rfdiffusion,
  title={De novo design of protein structure and function with RFdiffusion},
  author={Watson, Joseph L and Juergens, David and Bennett, Nathaniel R and Trippe, Brian L and Yim, Jason and Eisenach, Helen E and Ahern, Woody and Borst, Andrew J and Ragotte, Robert J and Milles, Lukas F and others},
  journal={Nature},
  volume={620},
  number={7976},
  pages={1089--1100},
  year={2023},
  publisher={Nature Publishing Group UK London}
}

@article{frey2025lab,
  title={Lab-in-the-loop therapeutic antibody design with deep learning},
  author={Frey, Nathan C and H{\"o}tzel, Isidro and Stanton, Samuel D and Kelly, Ryan and Alberstein, Robert G and Makowski, Emily and Martinkus, Karolis and Berenberg, Daniel and Bevers III, Jack and Bryson, Tyler and others},
  journal={bioRxiv},
  pages={2025--02},
  year={2025},
  publisher={Cold Spring Harbor Laboratory}
}

@article{joshi2019towards,
  title={Towards realistic individual recourse and actionable explanations in black-box decision making systems},
  author={Joshi, Shalmali and Koyejo, Oluwasanmi and Vijitbenjaronk, Warut and Kim, Been and Ghosh, Joydeep},
  journal={arXiv preprint arXiv:1907.09615},
  year={2019}
}

@article{dhurandhar2018explanations,
  title={Explanations based on the missing: Towards contrastive explanations with pertinent negatives},
  author={Dhurandhar, Amit and Chen, Pin-Yu and Luss, Ronny and Tu, Chun-Chen and Ting, Paishun and Shanmugam, Karthikeyan and Das, Payel},
  journal={Advances in neural information processing systems},
  volume={31},
  year={2018}
}

@article{pacesa2024bindcraft,
  title={BindCraft: one-shot design of functional protein binders},
  author={Pacesa, Martin and Nickel, Lennart and Schellhaas, Christian and Schmidt, Joseph and Pyatova, Ekaterina and Kissling, Lucas and Barendse, Patrick and Choudhury, Jagrity and Kapoor, Srajan and Alcaraz-Serna, Ana and others},
  journal={bioRxiv},
  pages={2024--09},
  year={2024},
  publisher={Cold Spring Harbor Laboratory}
}

@article{cho2025boltzdesign1,
  title={BoltzDesign1: Inverting All-Atom Structure Prediction Model for Generalized Biomolecular Binder Design},
  author={Cho, Yehlin and Pacesa, Martin and Zhang, Zhidian and Correia, Bruno and Ovchinnikov, Sergey},
  journal={bioRxiv},
  pages={2025--04},
  year={2025},
  publisher={Cold Spring Harbor Laboratory}
}

@article{anishchenko2021novo,
  title={De novo protein design by deep network hallucination},
  author={Anishchenko, Ivan and Pellock, Samuel J and Chidyausiku, Tamuka M and Ramelot, Theresa A and Ovchinnikov, Sergey and Hao, Jingzhou and Bafna, Khushboo and Norn, Christoffer and Kang, Alex and Bera, Asim K and others},
  journal={Nature},
  volume={600},
  number={7889},
  pages={547--552},
  year={2021},
  publisher={Nature Publishing Group UK London}
}

@article{goverde2023novo,
  title={De novo protein design by inversion of the AlphaFold structure prediction network},
  author={Goverde, Casper A and Wolf, Benedict and Khakzad, Hamed and Rosset, St{\'e}phane and Correia, Bruno E},
  journal={Protein Science},
  volume={32},
  number={6},
  pages={e4653},
  year={2023},
  publisher={Wiley Online Library}
}

@article{kosugi2022solubility,
  title={Solubility-aware protein binding peptide design using AlphaFold},
  author={Kosugi, Takatsugu and Ohue, Masahito},
  journal={Biomedicines},
  volume={10},
  number={7},
  pages={1626},
  year={2022},
  publisher={MDPI}
}

@article{dhariwal2021diffusion,
  title={Diffusion models beat gans on image synthesis},
  author={Dhariwal, Prafulla and Nichol, Alexander},
  journal={Advances in neural information processing systems},
  volume={34},
  pages={8780--8794},
  year={2021}
}

@article{ho2022classifier,
  title={Classifier-free diffusion guidance},
  author={Ho, Jonathan and Salimans, Tim},
  journal={arXiv preprint arXiv:2207.12598},
  year={2022}
}

@inproceedings{bansal2023universal,
  title={Universal guidance for diffusion models},
  author={Bansal, Arpit and Chu, Hong-Min and Schwarzschild, Avi and Sengupta, Soumyadip and Goldblum, Micah and Geiping, Jonas and Goldstein, Tom},
  booktitle={Proceedings of the IEEE/CVF Conference on Computer Vision and Pattern Recognition},
  pages={843--852},
  year={2023}
}

@article{reward1,
  title={Practical and asymptotically exact conditional sampling in diffusion models},
  author={Wu, Luhuan and Trippe, Brian and Naesseth, Christian and Blei, David and Cunningham, John P},
  journal={Advances in Neural Information Processing Systems},
  volume={36},
  pages={31372--31403},
  year={2023}
}

@article{reward2,
  title={Diffusion probabilistic modeling of protein backbones in 3d for the motif-scaffolding problem},
  author={Trippe, Brian L and Yim, Jason and Tischer, Doug and Baker, David and Broderick, Tamara and Barzilay, Regina and Jaakkola, Tommi},
  journal={arXiv preprint arXiv:2206.04119},
  year={2022}
}

@article{reward3,
  title={Monte Carlo guided diffusion for Bayesian linear inverse problems},
  author={Cardoso, Gabriel and Idrissi, Yazid Janati El and Corff, Sylvain Le and Moulines, Eric},
  journal={arXiv preprint arXiv:2308.07983},
  year={2023}
}

@inproceedings{reward4,
  title={Diffusion posterior sampling for linear inverse problem solving: A filtering perspective},
  author={Dou, Zehao and Song, Yang},
  booktitle={The Twelfth International Conference on Learning Representations},
  year={2024}
}

@article{singhal2025general,
  title={A general framework for inference-time scaling and steering of diffusion models},
  author={Singhal, Raghav and Horvitz, Zachary and Teehan, Ryan and Ren, Mengye and Yu, Zhou and McKeown, Kathleen and Ranganath, Rajesh},
  journal={arXiv preprint arXiv:2501.06848},
  year={2025}
}

@article{winkler2022stochastic,
  title={Stochastic control for bayesian neural network training},
  author={Winkler, Ludwig and Ojeda, C{\'e}sar and Opper, Manfred},
  journal={Entropy},
  volume={24},
  number={8},
  pages={1097},
  year={2022},
  publisher={MDPI}
}

@article{lipman2024flow,
  title={Flow matching guide and code},
  author={Lipman, Yaron and Havasi, Marton and Holderrieth, Peter and Shaul, Neta and Le, Matt and Karrer, Brian and Chen, Ricky TQ and Lopez-Paz, David and Ben-Hamu, Heli and Gat, Itai},
  journal={arXiv preprint arXiv:2412.06264},
  year={2024}
}

@article{kidger2021efficient,
  title={Efficient and accurate gradients for neural sdes},
  author={Kidger, Patrick and Foster, James and Li, Xuechen Chen and Lyons, Terry},
  journal={Advances in Neural Information Processing Systems},
  volume={34},
  pages={18747--18761},
  year={2021}
}

@inproceedings{winkler2024bridging,
  title={Bridging discrete and continuous state spaces: Exploring the Ehrenfest process in time-continuous diffusion models},
  author={Winkler, Ludwig and Richter, Lorenz and Opper, Manfred},
  booktitle={International Conference on Machine Learning},
  pages={53017--53038},
  year={2024},
  organization={PMLR}
}

@inproceedings{li2020scalable,
  title={Scalable gradients and variational inference for stochastic differential equations},
  author={Li, Xuechen and Wong, Ting-Kam Leonard and Chen, Ricky TQ and Duvenaud, David K},
  booktitle={Symposium on Advances in Approximate Bayesian Inference},
  pages={1--28},
  year={2020},
  organization={PMLR}
}

@article{wang2024training,
  title={Training free guided flow matching with optimal control},
  author={Wang, Luran and Cheng, Chaoran and Liao, Yizhen and Qu, Yanru and Liu, Ge},
  journal={arXiv preprint arXiv:2410.18070},
  year={2024}
}

@inproceedings{liu2023flowgrad,
  title={Flowgrad: Controlling the output of generative odes with gradients},
  author={Liu, Xingchao and Wu, Lemeng and Zhang, Shujian and Gong, Chengyue and Ping, Wei and Liu, Qiang},
  booktitle={Proceedings of the IEEE/CVF Conference on Computer Vision and Pattern Recognition},
  pages={24335--24344},
  year={2023}
}

@inproceedings{xie2026enhanced,
  title={Enhanced Diffusion Sampling: Efficient Rare Event Sampling and Free Energy Calculation with Diffusion Models},
  author={Xie, Yu and Winkler, Ludwig and Sun, Lixin and Lewis, Sarah and Foster, Adam and Jimenez-Luna, Jose and Hempel, Tim and Gastegger, Michael and Chen, Yaoyi and Zaporozhets, Iryna and others},
  booktitle={ICML 2026 Workshop on Structured Probabilistic Inference $\{$$\backslash$\&$\}$ Generative Modeling},
  year={2026}
}

@article{clark2023draft,
  title={Directly fine-tuning diffusion models on differentiable rewards},
  author={Clark, Kevin and Vicol, Paul and Swersky, Kevin and Fleet, David J},
  journal={arXiv preprint arXiv:2309.17400},
  year={2023}
}
\bibliographystyle{iclr_formatting/iclr2026_conference}


\newpage
\appendix

\section{Riemannian Gradient Implementation for $\mathbf{SO(3)}$}\label{app:riemannian_gradient_code}

As derived in Theorem~\ref{th:autograd}, we can efficiently modify existing autograd engines to calculate the $SO(3)$ Riemannian gradient. Specifically, we can wrap the tensor by an operation that acts like the identity in the forward pass and adjusts the gradient to coincide with the Riemannian gradient in the backward pass using Theorem~\ref{th:autograd}. In PyTorch, this can most conveniently be implemented with a backward hook as shown below. We emphasize that this implementation is completely general and ensures that the Riemannian gradient is seamlessly integrated in existing autograd functionality. 
\begin{lstlisting}
import torch
from scipy.spatial.transform import Rotation

def so3tensor(*args, **kwargs):
    x = torch.tensor(*args, **kwargs)

    def so3_backward_hook(grad):
        so3_grad = grad @ x.transpose(-1, -2)
        return so3_grad - so3_grad.transpose(-1, -2)

    if x.requires_grad:
        x.register_hook(so3_backward_hook)
    return x

x = so3tensor(Rotation.random(1).as_matrix(), requires_grad=True)
y = so3tensor(Rotation.random(1).as_matrix())

loss = (x - y).pow(2).sum()
so3_grad = torch.autograd.grad(loss, x)[0]
print(so3_grad)

>>tensor([[[ 0.0000,  0.8490,  0.1187],
           [-0.8490,  0.0000, -2.2494],
           [-0.1187,  2.2494,  0.0000]]], dtype=torch.float64)
    
\end{lstlisting}

\section{Expanded Related Work}\label{app:related_work}

\textbf{Backpropagation through folding models:} protein hallucination is a version of computational protein design which uses backpropagation through a folding model to its input sequence \cite{anishchenko2021novo,kosugi2022solubility,goverde2023novo,pacesa2024bindcraft}. Interestingly, even the most recent work \cite{cho2025boltzdesign1} based on AlphaFold3 avoids explicit backpropagation through its diffusion structure module and instead relies on the pair representation - a limitation that could potentially be overcome by our proposed methods. 

\textbf{Gradient descent and normalizing flows} has been explored in the explainability literature to generate counterfactual explanations \cite{joshi2019towards,dombrowski2021diffeomorphic, dombrowski2023diffeomorphic,dhurandhar2018explanations} although not for flow matching and diffusion models. 
\citet{dflow} proposes to differentiate through flows for controlled generation. Our work builds on this reference by generalizing it to matrix Lie groups, which is of high relevance for proteins. We also provide a detailed theoretical analysis of the method and derive an efficient adjoint state method. 
\citet{wang2024training} explores related ideas for matrix groups in the framework of optimal control. Specifically, the authors propose to add an additive control to the vector field of the flow. Our approach does not use a control but rather optimizes the initial condition following \cite{dflow}. We discuss the relationship to this reference in more detail in the appendix and compare in detailed numerical experiments to their approach. \citet{liu2023flowgrad} similarly relies on control variables but directly applies gradient descent to them. This reference does however not consider matrix groups.

\textbf{Adjoint state method on manifolds:} the adjoint state method on manifolds has been discussed in other works - however only in terms of charts \cite{lou2020neural, mathieu2020riemannian} or particular manifolds \cite{rezende2020normalizing, bacchio2023learning, albergo2022building}. The adjoint state method derived in our work is applicable to any matrix Lie group and is particularly efficient for the group $SE(3)$ which plays a crucial role the frame-based protein backbone representation.

\textbf{Guidance:} a widely used method to bias diffusion and flow-matching models towards certain desiderata is guidance. There exist various flavors of it, such as classifier-based guidance \cite{dhariwal2021diffusion}, classifier-free guidance \cite{ho2022classifier}, and universal guidance \cite{bansal2023universal}. There has also been substantial recent interest in biasing the generation with additional reward functions~\cite{reward1, reward2, reward3, reward4, xie2026enhanced, singhal2025general} often relying on sequential Monte Carlo. However, these methods struggle for guidance potentials that are sensitive to fine-grained details of the final sample, such as force fields or certain biochemical properties, as they tend to rely on few-shot denoising and require carefully hyperparameter finetuing of their starting point and time-dependent weighting factor. We discuss the limitations of guidance in more detail in Appendix~\ref{app:limitations_of_guidance}.

\textbf{Relation to \cite{wang2024training}:}
a closely related reference is \cite{wang2024training} which outlines a very nice approach to maximize a terminal reward using deterministic optimal control \cite{winkler2022stochastic}. Specifically, the authors propose to add an additive control $\theta_\tau$ to the vector field of the flow
\begin{align}
    \frac{dx_\tau}{d\tau} = v_\theta(x_\tau, \tau) && \rightarrow && \frac{dx_\tau}{d\tau} = v_\theta(x_\tau, \tau) + \theta_\tau
\end{align}
The authors then propose to optimize the control with respect to a (regularized) loss
\begin{align}
    \mathcal{L}(x^\theta_1) + d(x_1^\theta, x^0_1) + \frac12 \int_0^1 d\tau \, ||\theta_\tau||^2
\end{align}
where $x_1^\theta$ corresponds to the terminal value under the control $\theta_\tau$ and $x^0_1$ is the terminal value for no control. Furthermore, $d(\cdot, \cdot)$ denotes a distance. 

The optimal control approach defines control terms $\theta_\tau$ at intermediate time steps $\tau$ during the generative process.
By construction, the control thus acts upon the generation not only at the initial condition $\tau=0$ but also for $\tau >0$. 
A change in initial condition $x_0 \to x_0 + \delta x_0$ can then only be interpreted as a control if the latter is not a function $\forall \tau$ but rather concentrates all modelling capacity at only $\tau=0$. In the continuous time setting this would imply a control proportional to a Dirac impulse $\delta(\tau)$ in time, seen as follows
\begin{align}
    x_1 = x_0 + \delta x_0 + \int_0^1 v(x_\tau) d\tau  = x_0 + \int_0^1 (v_\theta(x_\tau) + \delta(\tau) \delta x_0 ) \, d \tau 
\end{align}
and thus implying a distributional control $\theta_\tau =  \delta(\tau) \delta x_0$. Note that this distributional control has highly undesirable properties from a numerics standpoint, i.e., it  is necessarily divergent for $\tau=0$ and vanishing for all other values of flow time $\tau$.
Intuitively, it would require the control to voluntarily constrain all its modelling capacity to the time of the initial condition $\tau=0$, something that is hard to achieve in practice with unconstrained optimization.

Diffusion and flow models are trained to map a simple base space variable to a target distribution.
Changing the base space variable ensures that the samples remain on the learned manifold in the target space.
This can not be guaranteed if the trajectory is changed at intermediate points.
The control parameters can exploit the loss function $\mathcal{L}(x_1^\theta)$ while moving the sample off manifold.
To counteract such movements, a distance loss term $d(x_1^\theta, x_1^0)$ in the target space is applied but this but simultaneously constraints how much the original loss function can be optimized.
Diffeomorphic optimization solves both problems conveniently at the same time by allowing free, unconstrained movement over the manifold while also remaining on the learned manifold at all times.

\section{Proofs}\label{app:proofs}

\subsection{Proof of Theorem~\ref{th:on_manifold}}\label{app:proof_on_manifold}

\begin{lemma} \label{lemma1}
    For $g:M \to N$, it holds that
    \begin{align}
        g( \exp_p(\lambda  v)) = \exp_{g(p)} (\lambda \, dg_p ( v ) + \mathcal{O}(\lambda^2))
    \end{align}
    for any $p\in M$, $v \in T_pM$, and small $\lambda \in \mathbb{R}$.
\end{lemma}
\begin{proof}
    For proving this statement, it is useful to consider normal coordinates. For $p \in M$ and a neighborhood $U$ of $0 \in T_p M$, these coordinates are given by the chart $\psi: U \subset T_pM \to M$ with $\psi(v) = \exp_p(v)$. The coordinates $y^\alpha$ of a point $q \in M$ in this chart are obtained by
    \begin{align}
        q = \exp_p(y^\alpha e_\alpha) \,
    \end{align}
    where $e_\alpha$ is an orthonormal basis of $T_pM$ and we have used the Einstein summation convention. This implies that
    \begin{align}
        y^\alpha(\exp_q(v)) = v^\alpha
    \end{align}
    Using normal coordinates for both $M$ and $N$, we can expand the function $g$ using Taylor's theorem
    \begin{align}
        \widehat{g( \exp_p(\lambda  v))} =  g^\alpha(y^\beta) = g^\alpha(\lambda v^\beta) = g^\alpha(0) + \lambda \frac{\partial g^\alpha}{\partial y^\beta} v^\beta + \mathcal{O}(\lambda^2) = \widehat{\exp_{g(p)}} (\lambda \, \widehat{dg_p} ( v ) + \mathcal{O}(\lambda^2)) 
    \end{align}
    where we have used that the connection $\Gamma$ vanishes in normal coordinates and the hat symbol denotes the coordindate representation. Since the left and the right hand side are written in terms of covariant objects, the result stated in the theorem follows.
\end{proof}

\begin{lemma}\label{lemma2}
    Let $g: Z \to X$ be a diffeomorphism with $Z$ being a Riemannian manifold with metric $G$. Then, it holds that
    \begin{align}
        dg_z(\textrm{grad}_z^G \mathcal{L} \circ g) = \textrm{grad}^{\tilde{G}}_{g(z)} \mathcal{L}
    \end{align}
    where $\tilde{G} = g_* G$ is the pushforward metric of $G$.
\end{lemma}
\begin{proof}
    This statement is easily shown in coordinates. Let $z^\mu$ denote coordinates on $Z$ and $x^\alpha=g^\alpha(z)$. Then, the coordinate representation of the pushforward metric $\tilde{G}_{\alpha \beta}$ follows by
    \begin{align}
        G_{\mu \nu} dz^{\mu} dz^{\nu} =  G_{\mu \nu}  \frac{\partial z^\mu}{\partial x^\alpha} \frac{\partial z^\nu}{\partial x^\beta} dx^{\alpha} dx^{\beta} = \tilde{G}_{\alpha \beta}  dx^{\alpha} dx^{\beta}
    \end{align}
    In these coordinates, $dg_z(\textrm{grad}_z^G \mathcal{L} \circ g)$ is given by
    \begin{align}
        \frac{\partial x^\alpha}{\partial z^\mu} G^{\mu \nu} \frac{\partial (\mathcal{L} \circ g)}{\partial z^\nu} = \frac{\partial x^\alpha}{\partial z^\mu} G^{\mu \nu} \frac{\partial x^\sigma}{\partial z^\nu}\frac{\partial \mathcal{L}}{\partial x^\sigma} = \tilde{G}^{\alpha \sigma} \frac{\partial \mathcal{L}}{\partial x^\sigma}
    \end{align}
    which corresponds to the coordinate representation of the right hand side.
\end{proof}

Theorem~\ref{th:on_manifold} follows almost immediately using the the two lemmas:
\begin{proof}
    By Lemma~\ref{lemma1}, it follows that
    \begin{align}
         g(\exp_z(- \lambda \, \textrm{grad}^G_z \mathcal{L} \circ g ) )= \exp_{g(z)} ( - \lambda \, dg_z (\textrm{grad}^G_z \mathcal{L}) + \mathcal{O}(\lambda^2) ) \,.
    \end{align}
    The right hand side can then be rewritten by Lemma~\ref{lemma2} as
    \begin{align}
        \exp_{g(z)} ( - \lambda \, dg_z (\textrm{grad}^G_z \mathcal{L}) + \mathcal{O}(\lambda^2) ) = \exp_{g(z)} ( - \lambda \, \textrm{grad}^{\tilde{G}}_{g(z)} \mathcal{L} + \mathcal{O}(\lambda^2) )
    \end{align}
    which shows the stated result.
\end{proof}

\subsection{Proof of Theorem~\ref{th:autograd}}\label{app:proof_of_autograd}
We consider the definition of the Riemannian gradient on the Lie algebra
\begin{align}
    \partial^a f(R) &= \frac{d}{d\tau} f(\exp(\tau T^a) R)|_{\tau=0} \\
    &= \frac{df}{dR_{\alpha \beta}} \frac{d (\exp(\tau T^a) R)_{\alpha \beta}}{d \tau} |_{\tau=0} \\
    &= \frac{df}{dR_{\alpha \beta}} (T^a R)_{\alpha \beta} \\
    &= \frac{df}{dR_{\alpha \beta}} T^a_{\alpha \sigma} R_{\sigma \beta} \\
    &= \frac{df}{dR_{\alpha \beta}} T^a_{\alpha \sigma} (R^T)_{\beta \sigma} && | \quad T^a = - (T^a)^T\\
    &= - \frac{df}{dR_{\alpha \beta}} T^a_{\sigma \alpha } (R^T)_{\beta \sigma}  \\
    &= - \textrm{tr} \left( T^a \frac{df}{dR} R^T \right ) \\
    &= - \frac{1}{2} \textrm{tr} \left( T^a \, 2 \frac{df}{dR} R^T \right ) \\
    &= \langle T^a, 2 \frac{df}{dR} R^T \rangle
\end{align}
where we have used the Einstein summation convention and the definition of inner product $\langle A, B \rangle  = \frac12 \textrm{tr}A^T B$ for Lie algebra elements $A,B \in \mathfrak{so}(3)$. Recall that the Riemannian gradient is given by
\begin{align}
    \partial f \equiv \sum_a T^a \partial_a f(R) 
\end{align}
Due to the orthonormality of the generators, $\langle T^a, T^b\rangle = \delta^{ab}$ and denoting the antisymmetric part $[M]_A = \frac12 (M - M^T)$ of a matrix $M$, it thus holds that
\begin{align}
    \partial f = 2 \left[ \frac{df}{dR} R^T \right]_A \,,
\end{align}
The antisymmetric component arises in this step because for an arbitrary matrix $M$, we have $\langle T^a, M\rangle = \langle T^a, [M]_A \rangle$.

\subsection{Proof of Theorem~\ref{th:adjoint_state}}\label{app:proof_of_adjoint_state}
The adjoint state is given by
\begin{align}
    A_\tau = \nabla_{Z_\tau} \mathcal{L} 
\end{align}
for which we want to derive a differential equation for its time evolution.

For this, we define the time evolution operator
\begin{align}
    T_\epsilon: SO(3) \to SO(3)\,, && Z_\tau \mapsto Z_{\tau + \epsilon} = \exp(\epsilon V_\theta(Z_\tau)) \, Z_\tau \,.
\end{align}
We use the chain rule \eqref{eq:chain_rule_lie_algebra} to derive that
\begin{align}
    A^a_\tau = \nabla^a_{Z_{\tau}} \mathcal{L} &=  \sum_b \nabla^b_{Z_{\tau + \epsilon}} \mathcal{L} \; \mathfrak{D}^{ba} T_\epsilon \\
    &= \sum_b A^b_{\tau+\epsilon} \;\mathfrak{D}^{ba} T_\epsilon \,.\label{eq:adjoint_state_intermediate}
\end{align}
The differential \eqref{eq:so3_jacobian} of the time evolution operator is given by:
\begin{align}
    \mathfrak{D}^{ba} T_\epsilon &= \langle T^b, \nabla_{Z_\tau}^a T_\epsilon \; T_\epsilon^\top \rangle  \\
    &= \langle T^b, \nabla_{Z_\tau}^a(e^{\epsilon V_\theta(Z_\tau)} Z_\tau) \; Z_\tau^\top e^{-\epsilon V_\theta(Z_\tau)} \rangle \,, 
\end{align}
where we have used that $V_\theta \in \mathfrak{so}(3)$ is antisymmetric. We now expand this expression up to quadratic order in the step size $\epsilon$ to obtain
\begin{align}
   \mathfrak{D}^{ba} T_\epsilon &=  \langle T^b, \nabla_{Z_\tau}^a(\;(\mathbb{I}+\epsilon V_\theta(Z_\tau)) \, Z_\tau) \; Z_\tau^\top (\mathbb{I}-\epsilon V_\theta(Z_\tau)) \,\rangle + \mathcal{O}(\epsilon^2) \\
   &= \delta^{ba} + \epsilon \langle T^b, [V_\theta, T^a]\rangle + \epsilon \langle T^b , \nabla^a_{Z_\tau} V_\theta(Z_\tau)\rangle + \mathcal{O}(\epsilon^2)   
\end{align}
We now use the fact that
\begin{align}
    \langle T^b, [V_\theta, T^a]\rangle 
    = -\frac{1}{2} \textrm{tr} \, T^b \left(V_\theta T^a - T^a V_\theta \right)
    = -\frac{1}{2} \textrm{tr} \, T^a \left(T^b V_\theta - V_\theta T^b\right) =  \langle T^a, [T^b, V_\theta] \rangle \,.
\end{align}
Expanding $V_\theta = \sum_c V_\theta^c T^c$ and similarly for $[T^b, V_\theta] \in \mathfrak{so}(3)$, we then obtain that the differential is given by
\begin{align}
    \mathfrak{D}^{ba} T_\epsilon &=  \delta^{ba} - \epsilon [V_\theta, T^b]^a + \epsilon \nabla^a_{Z_\tau} V_\theta(Z_\tau)^b + \mathcal{O}(\epsilon^2) 
\end{align}
Plugging this result into the chain rule for the adjoint state \eqref{eq:adjoint_state_intermediate}, we obtain
\begin{align}
     A^a_\tau = A^a_{\tau+\epsilon} - \epsilon [V_\theta, A_{\tau+\epsilon}]^a + \epsilon \sum_b A^b_{\tau+\epsilon} \, \nabla^a_{Z_\tau} V_\theta(Z_\tau)^b + \mathcal{O}(\epsilon^2) \,.
\end{align}
We now rearrange this to isolate the finite time difference on the right-hand-side
\begin{align}
    \tfrac{1}{\epsilon} \left( A^a_{\tau + \epsilon} - A^a_\tau \right) = [V_\theta, A_{\tau+\epsilon}]^a -  \sum_b A^b_{\tau+\epsilon} \, \nabla^a_{Z_\tau} V_\theta(Z_\tau)^b + \mathcal{O}(\epsilon) \,.
\end{align}
Taking the limit $\epsilon \to 0$ gives
\begin{align}
    \frac{d}{d\tau} A^a_\tau = [V_\theta, A_{\tau}]^a -  \sum_b A^b_{\tau} \, \nabla^a_{Z_\tau} V_\theta(Z_\tau)^b \,.
\end{align}
Multiplying by the generator $T^a$ and summing over the index $a$, we then obtain:
\begin{align}
    \frac{d}{d\tau} A_\tau = [V_\theta, A_{\tau}] -  \sum_b A^b_{\tau} \, \nabla_{Z_\tau} V_\theta(Z_\tau)^b \,,
\end{align}
which is the claimed time-evolution equation of the theorem.

\section{Lightning Review of Differential Geometry and Lie Groups} \label{app:sec:gradientdescentonliegroups}
\paragraph{Right multiplication: }Lie groups are smooth manifolds endowed by an additional group multiplication. There is therefore a natural diffeomorphism given by right multiplication
\begin{align}
    R_g: G \to G\,, && R_g: h \mapsto h g \,,
\end{align}
for $h,g \in G$. Let us denote the unit element of the group by $e$. Since right multiplication $R_g$ is a diffeomorphism, its differential 
$$(dR_g)_e: T_e G \to T_g G $$ 
is a linear isomorphism (linear bijective map). We can therefore uniquely identify tangent vectors $v \in T_g G$ and Lie algebra elements $\tilde{v} \in \mathfrak{g} \simeq T_eG$. 

\paragraph{Matrix groups and right multiplication:} For matrix groups, the differential $(dR_g)_e: T_e G \to T_gG$ takes a particularly simple form
\begin{align}
    (dR_g)_e v = v g \label{eq:right_isomorphism}
\end{align}
where we take the standard matrix product of the Lie algebra element $v$ and the group element $g$ on the right hand side. This statement can be easily checked by noticing that the differential acts on any function $f: G \to \mathbb{R}$ by
\begin{align}
    [(dR_g)_e v]f = \frac{d}{dt} f (\underbrace{\gamma_v(t) g}_{\equiv \tilde{\gamma}(t)})  \bigg|_{t=0} \,,
\end{align}
where $\gamma_v$ denotes the curve associated with $v \in T_e G$, i.e. $\gamma_v(0)=e$ and $\frac{d}{dt} \gamma_v(t)|_{t=0} = v$. 
As a result, the curve $\tilde{\gamma}$ obeys 
\begin{align}
    \tilde{\gamma}(0) = g \,, &&
    \frac{d}{dt} \tilde{\gamma}(t) |_{t=0} = \frac{d}{dt} \gamma(t) |_{t=0} \, g = v g \,,
\end{align}
and is therefore a curve associated with $v g \in T_gG$ as claimed.

\paragraph{Killing form:} we first define the adjoint map
\begin{align}
    \textrm{ad}_u(v) = [u, v]
\end{align}
where $u,v \in \mathfrak{g}$. The adjoint map is linear and a Lie algebra homomorphism 
\begin{align}
    \textrm{ad}_{[u,v]} = [\textrm{ad}_u, \textrm{ad}_v] \,.
\end{align}
Then the Killing form is given by
\begin{align}
    B(u, v) = \textrm{Tr}(\textrm{ad}_u \circ \textrm{ad}_v)
\end{align}
for $u,v \in \mathfrak{g}$ is a bilinear symmetric form that is non-degenerate if the group is semi-simple and compact. We can use the Killing form to define a inner product for $T_eG$ by
\begin{align}
    \langle u, v \rangle_e = -B(u, v) \,.
\end{align}
The negative sign ensures positive definiteness.

\paragraph{Riemannian metric:} Often it is more convenient to work with the Lie algebra instead of the various tangent spaces. For example, an inner product $\langle \cdot, \cdot \rangle_e$ on the Lie algebra induces a Riemannian inner product on the group $G$ by 
\begin{align}
    \langle v, w \rangle_g \equiv \langle (dR_{g^{-1}})_g v, (dR_{g^{-1}})_g w \rangle_e = \langle v g^{-1}, w g^{-1} \rangle_e \,, && \text{for} \; v, w \in T_gG \,,
\end{align}
where we have used that $dR_{g^{-1}} = (dR_g)^{-1}$ to pull tangent vectors back to the Lie algebra. Using $(dR_g)^{-1} = dR_{g^{-1}}$, it immediately follows that the inner product is right-invariant, i.e., 
\begin{align}
   \langle (dR_g)_e \xi, (dR_g)_e \eta \rangle_g = \langle \xi, \eta \rangle_e \label{eq:rightinv}
\end{align}
for Lie algebra elements $\xi, \eta \in T_eG$. 
For $SO(3)$, we choose antisymmetric generators $T^a$ that are orthonormal with respect to the inner product induced by the Killing form $\langle T^a, T^b \rangle = \frac12 \textrm{tr} (T^a)^\top T^b = \delta^{ab}$.

\paragraph{Riemannian gradient:} Recall that the Riemannian gradient $\textrm{grad}_g f \in T_g G$ is the the unique element of the tangent space which obeys
\begin{align}
    \langle \textrm{grad}_g f, v\rangle_g = df_g \, v
\end{align}
for all $v \in T_g G$. Due to the isomorphism between the tangent space $T_g G$ and the Lie algebra $\mathfrak{g} \simeq T_eG$, the Riemannian gradient  $\textrm{grad}_g f \in T_gG$ uniquely corresponds to the Lie algebra element 
\begin{align}
\nabla f \equiv (dR_{g^{-1}})_e \, \textrm{grad}_g f \in \mathfrak{g} \,. \label{eq:liealgebragradient}
\end{align}
It is often more convenient to work with this Lie algebra representative of the Riemannian gradient. We will now derive a simple expression for this Lie algebra gradient. By definition of the Riemannian gradient, it holds that
\begin{align}
    \langle \textrm{grad}_g f, (dR_g)_e v \rangle_g = df_g \left((dR_g)_ev \right) \,, \label{eq:gradlie}
\end{align}
for any $v \in \mathfrak{g}$. We can invert \eqref{eq:liealgebragradient} to obtain
\begin{align}
    \textrm{grad}_g f = (dR_g)_e \nabla f  \label{eq:liegrad2}
\end{align}
Using this, we can rewrite \eqref{eq:gradlie} as follows
\begin{align}
    \langle (dR_g)_e \nabla f, (dR_g)_e v \rangle_g = \langle \nabla f, v \rangle_e \,,
\end{align}
where we have used the right-invariance of the metric \eqref{eq:rightinv}.
We thus conclude that
\begin{align}
  \langle \nabla f, v \rangle_e = df_g \left((dR_g)_ev \right) 
\end{align}
By definition of the right hand side, it therefore holds for a matrix group that
\begin{align}
    \langle \nabla f, v \rangle_e = df_g (v g) =  \frac{d}{dt} f(\exp(t v) g) \bigg|_{t=0} \,. \label{eq:prev}
\end{align}
We can expand the Lie algebra representative in terms of generators $\nabla f = \sum_a T^a \nabla^a f$ whose components are given by
\begin{align}
    \nabla^a f = \langle \nabla f, T^a \rangle_e =  \frac{d}{dt} f(\exp(t T^a) g) \bigg|_{t=0} \,,
\end{align}
where we have used \eqref{eq:prev} and assumed that the generators are chosen to be orthonormal, i.e., $\langle T^a, T^b \rangle_e = \delta^{ab}$. For brevity, we typically refer to the Lie algebra representative of the Riemannian gradient $\nabla f$  as simply the Riemannian gradient in the main part of the paper.

\paragraph{Exponential map:} For a Lie group, the exponential map takes the form
\begin{align}
    \exp_g(v) = \exp( (dR_{g^{-1}})_g v) \, g
\end{align}
for $g \in G$ and $v \in T_gG$ where $\exp: \mathfrak{g} \to G$ is the matrix exponential for matrix Lie groups. We note that, by the Lie algebra expression of the Riemannian gradient \eqref{eq:liegrad2}, this implies that gradient descent on a Lie group with learning rate $\lambda \in \mathbb{R}$ and loss $\mathcal{L}: G \to \mathbb{R}$ is given by
\begin{align}
    g^{i+1} = \exp( -\lambda \nabla \mathcal{L}) g^{i} \,.
\end{align}
This result is used extensively in the main part of the paper.

\paragraph{Differential:} Right multiplication $R_g: G \to G$ induces a isomorphism between the tangent spaces $T_gG$ and the Lie algebra $\mathfrak{g}$. Consider the differential $df_g: T_gG \to T_{f(g)}G$ of the map $f:G\to G$. Using the right multiplicative isomorphism \eqref{eq:right_isomorphism}, we can write any tangent space element in terms of a Lie algebra element
\begin{align}
    T_gG \ni \Omega G \leftrightarrow \omega \in \mathfrak{g}
\end{align}
Specifically, we can define a Lie algebra representative of the differential
\begin{align}
    \mathfrak{D} f_g \equiv d(R_{f(g)^{-1}}\circ f_g \circ R_g)_e: T_e G \simeq \mathfrak{g} \to T_eG \simeq \mathfrak{g}
\end{align}
Using the fact that the right isomorphism amounts to right multiplication \eqref{eq:right_isomorphism}, we can derive a explicit expression for this representative
\begin{align}
    \mathfrak{D} f_g(\omega) &= df_g(\omega g) f(g)^\top \\
    &= \frac{d}{dt} f(\exp(t \omega)g) \big|_{t=0} f(g)^\top \\
    &= \frac{df}{dg}(\omega g) f(g)^\top \\
    &= \sum_a \omega_a \frac{df}{dg}(T^a g) f(g)^\top \\
    &= \sum_a \omega_a \frac{d}{dt} f(\exp(t T^a g)\big|_{t=0} f(g)^\top \\
    &=\sum_a \omega_a \nabla^a f(g) \; f(g)^\top \,.
\end{align}
Since $\mathfrak{D}f_g(\omega) \in \mathfrak{g}$, we can expand it as $\mathfrak{D}f_g(\omega) = \sum_b T^b \mathfrak{D}^b f_g(\omega)$. Hence, it holds that
\begin{align}
    \mathfrak{D}^bf_g(\omega) = \sum_a \langle T^b, \nabla^a f(g) \; f(g)^\top \rangle \, \omega_a \equiv \sum_a \mathfrak{D}f^{ba}_g \, \omega_a \,,
\end{align}
with 
\begin{align}
    \mathfrak{D}f^{ba}_g \equiv \langle T^b, \nabla^a f(g) \; f(g)^\top \rangle \,.\label{eq:so3_jacobian}
\end{align}
We can think of $\mathfrak{D}f^{ba}_g$ as being the matrix representation of the differential in the generator basis $T^a$.

\paragraph{Chain rule:} Let's also consider a function $F: G \to \mathbb{R}$ and its composition $F \circ f: G \to \mathbb{R}$. Then it holds that
\begin{align}
    \nabla^a (F \circ f)_g = \sum_b \nabla^b F_{f(g)} \, \mathfrak{D}f^{ba}_g \,.\label{eq:chain_rule_lie_algebra}
\end{align}
This follows directly from 
\begin{align}
    F \circ f \circ R_g = F \circ R_g \circ R_{g^{-1}} \circ f \circ R_g \,.
\end{align}
Taking the differential of this and using the standard chain rule of the differential then gives
\begin{align}
    d( F \circ R_{f(g)} \circ R_{f(g)^{-1}} \circ f \circ R_g)_e = d(F \circ R_{f(g)})_e \, d(R_{f(g)^{-1}} \circ f \circ R_g)_e = \nabla F_{f(g)}  \; \mathfrak{D}f_g \,.
\end{align}


\section{SO(3) Conventions}
Any group element $g \in SO(3)$ can be written as
\begin{align}
    g = \exp( A)
\end{align}
with $A \in \mathfrak{so}(3)$ taking value in the Lie algebra
\begin{align}
    \mathfrak{so}(3) = \{ A^\intercal = - A \, | \, A \in \mathbb{R}^{3,3}  \} \,.
\end{align}
A Lie algebra is, in particular, a vector space. We will choose the basis $T^a$ with $a\in\{1, 2, 3\}$ with
\begin{align}
    T^1 = \begin{bmatrix}
0 & 0 & 0\\
0 & 0 & -1 \\
0 & 1 & 0
\end{bmatrix} \,, &&  T^2 = \begin{bmatrix}
0 & 0 & 1\\
0 & 0 & 0 \\
-1 & 0 & 0
\end{bmatrix}  \,, && T^3 = \begin{bmatrix}
0 & -1 & 0\\
1 & 0 & 0 \\
0 & 0 & 0
\end{bmatrix} \,.
\end{align}
In the context of Lie theory, the basis vectors $T^a$ are also referred to as the generators of the Lie algebra. In particular, any antisymmetric three-by-three matrix can be written as a linear combination of the generators. Furthermore, the generators obey the following commutation relations
\begin{align}
    [T^a, T^b] = \sum_{c=1}^3 \epsilon^{abc} T^c \,,
\end{align}
where  $\epsilon^{abc}$ denotes Levi-Civita symbol. It can be checked that 
\begin{align}
    \Tr (T^a) &= 0 \,, \\
    \Tr (T^a T^b) &= - 2 \delta^{ab} \,, \label{eq:orthogonality} 
\end{align}
where $\delta^{ab}$ is the Kronecker symbol. Using \eqref{eq:orthogonality}, we can equip the Lie algebra with an inner product
\begin{align}
    \langle A, B \rangle = \frac{1}{2} \Tr (A^\intercal B) = \sum_{a=1}^3 A^a B^a \,,
\end{align}    
where we have used that we can express an arbitrary element $A$ of the Lie algebra $\mathfrak{so}(3)$ as $A = \sum_{a=1}^3 A^a T^a$. Note that our basis $T^a$ is orthonormal with respect to this inner product. 

\section{Limitations of Guidance}\label{app:limitations_of_guidance}
Guidance is a technique to bias the generation process, i.e., instead of sampling from the model distribution $p$, one aims to sample from
\begin{align}
    p_{\textrm{tilted}}(x) = p(x) \frac1{Z_{\textrm{guide}}}e^{-\beta E_{\textrm{guide}}(x)}
\end{align}
where $E_\textrm{guide}$ denotes the guidance energy of interest. For score based models, this amounts to a modification of the score
\begin{align}
    \nabla_x \log p_{\textrm{tilted}}(x) = \nabla_x \log p(x) - \beta \,\nabla_x E_{\textrm{guide}}(x) \,. 
\end{align}
A challenging aspect of guidance is diffusion involves a one-parameter family of density which is typically parameterized in terms of diffusion time $t$ or noise level $\sigma(t)$. For many guidance energies, such as the Rosetta energy function, it is highly non-trivial to calculate the score of the guidance density at time $t$,
\begin{align}
 p_{\textrm{guide}}(x_t) = \frac1{Z_{\textrm{guide}}} \int p(x_0 | x_t) e^{-\beta E_{\textrm{guide}}(x_0)} dx_0 \,,
\end{align}
as it involves an untractable marginalization over the unnoisy datasample $x_0$. Here, $p(x_0|x_t)$ denotes the backward kernel, i.e. the density of obtaining the unnoisy data sample $x_0$ when integrating the reverse SDE starting from $x_t$ at diffusion time $t$. The fundamental challenge of all guidance schemes is to approximate the above marginalization in a suitable manner for the application at hand. This approximation becomes particularly challenging in the limit of $\beta \to \infty$ which corresponds to sampling minimizers of the guidance energy $E_{\textrm{guide}}$. In this limit, the integrand is heavily dominated by the minimizers (and its neighborhoods for very large but finite $\beta$ as used in numerical experiments). As such, one has to estimate the conditional density $p(x^*|x_t)$ for the minimizers $x^*$ with high precision. For this reason, guidance cannot be expected to work well for minimizing generic guidance energies.

Let us consider the simpler case in which we do not want to minimize the guidance energy, i.e., we do not consider the challenging $\beta \to \infty$ limit. In this case, there are various strategies to approximate the marginalization. None of them are perfect, and each comes with a certain set of tradeoffs. We will only focus on methods which work for a pretrained unconditional model as this is the relevant setting for diffeomorphic optimization. In principle, one could use Monte-Carlo to estimate the guided score
\begin{align}
    \nabla_{x_t} \log p_{\textrm{guide}}(x_t) \approx \nabla_{x_t} \log \frac{1}{N} \sum_{i=1}^N e^{-\beta E_{\textrm{guide}}(x_0^{(i)}(x_t))}
\end{align}
with $x^{(i)}_0 \sim p(x_0|x_t)$. For this, we have to repeatedly integrate the reverse SDE starting from $x_t$ and backpropagate through it. There are methods for refined backpropagation through SDEs, such as suitable generalizations of the adjoint state method\cite{kidger2021efficient, li2020scalable} but we are not aware of any work in the protein space applying these techniques. This is possibly due to the considerable technical difficulty, such as the appearance of Stratonovich stochastic integrals, as well as scaling concerns to larger problem sizes. Furthermore, one would need to derive generalizations of these methods for matrix Lie groups.

Universal guidance~\cite{bansal2023universal} is a technique inspired by classifier guidance~\cite{dhariwal2021diffusion}. In this setting the marginalization is approximated by a zeroth-order saddle point approximation around the expectation value $\hat{x}_0(x_t) = \mathbb{E}[x_0 | x_t] $ of the kernel $p(x_0|x_t)$, i.e., 
\begin{align}
 p_{\textrm{guide}}(x_t) \approx \frac1{Z_{\textrm{guide}}} \int \delta(\hat{x}_0(x_t)-x_0) \,  e^{-\beta E_{\textrm{guide}}(x_0)} dx_0 = \frac1{Z_{\textrm{guide}}} e^{-\beta E_{\textrm{guide}}(\hat{x}_0(x_t) )} \,,
\end{align}
which can be a good approximation at low noise levels (small diffusion time) but tends to be very poor at high noise. To alleviate this problem, guidance is often combined with a heuristic schedule which switches the guidance score on when the diffusion time is lower than some threshold (often in an adiabatic manner). However, this results in a fundamental tension: we want the guidance signal to kick in as early as possible in the generation process to meaningfully change the selected mode of the sample. On the other hand, we do not want to start the guidance too early as the zeroth-order saddle point approximation would be very poor. Whether such a delicate balance can be struck often heavily depends on the problem setting and requires careful (and costly) hyperparameter tuning. As a general rule of thumb, energy functions that depend only on the coarse grain details of the sample $x_0$ seem to be easier to deal with. Examples include the classification of image types or the sentiment of a sentence. However, energy functions in physics (which are highly relevant in the protein space) are typically not of this type. For example, the Rosetta energy is extremely sensitive to the repulsive component of the atomic interactions resulting in numerically destabilizing energy values even if the global fold and secondary structure of the predicted unnoisy protein is roughly correct.

Recently, there has been considerable interest in combining diffusion models with Sequential Monte Carlo (SMC) to facilitate sampling from the tilted distribution $p_{\textrm{tilted}}$. We base our discussion on the very recent \cite{singhal2025general} which introduces Feynman-Kac steering unifying a number of previous approaches in a common theoretical framework. The basic idea is to modify the transition kernel $p(x_{t-1}|x_t)$ of the reverse diffusion
\begin{align}
    p(x_{t}|x_{t+1}) \to p(x_{t}|x_{t+1}) G(x_1, \dots, x_{t}) \label{eq:tilted_trans_kernel}
\end{align}
by a guidance potential $G$ which fullfills the consistency condition
\begin{align}
    \prod_{t=1}^0 G(x_1, \dots, x_{t}) = \frac1{Z_{\textrm{guide}}}e^{-\beta E_{\textrm{guide}}(x_0)} \,.
\end{align}
This ensures that the joint distribution of the sampling trajectory $(x_1, \dots, x_t, \dots x_0)$ is given by
\begin{align}
     p_{\textrm{guide}}(x_1, \dots, x_0) = p(x_1, \dots, x_0) \frac1{Z_{\textrm{guide}}}e^{-\beta E_{\textrm{guide}}(x_0)}
\end{align}
and thus has the desired marginal $p_{\textrm{guide}}(x_0)$. In principle, any guidance potential satisfying the consistency condition will lead to asymptotic gurantees but in practice a careful choice of potential is vital to avoid prohibitive variance. This is because, at each step of the reverse diffusion one uses importance sampling to reweight the sample $x_t$ according to the tilted transition kernel \eqref{eq:tilted_trans_kernel}. In practice, one often has to choose guidance potentials $G$ that involve challenging marginalization. For example, a widely used guidance potential is the expected guidance energy $E_{guide}$. Thus this approach faces analogous challenges to universal guidance, particularly in the $\beta \to \infty$ limit.

Another important point of differentiation between guidance and diffeomorphic optimization is that the latter is particularly suited in applications in which we want to improve a particular sample $x_0$ with respect to some cost function. For guidance, we would need to incorporate a constraint in the guidance potential that ensures similarity to $x_0$. For diffeomorphic optimization, we perform gradient descent updates that iteratively refine the sample. 

\section{Complexity and Efficiency of Gradient Estimations through Flows}
A forward pass evaluates the network layer by layer, applying linear transformations and nonlinearities to propagate an input through the model. Its cost is determined by the size and number of layers, with matrix multiplications dominating the total FLOPs. 
A general backward pass is more expensive because it must compute gradients with respect to all parameters. 
This involves propagating gradients backward through each layer and forming gradient tensors for every weight matrix in each layer and the layers input for further backward propagation. These extra operations roughly double the compute relative to the forward pass, and they also require storing all intermediate activations, which increases memory consumption and can become a limiting factor on modern hardware.

A backward pass that computes gradients only with respect to the input is cheaper. 
It is comparable to a single backward pass as the parameter gradients do not have to be computed.
The computational effort becomes close to that of the forward pass, with only a small overhead for the backward sweep. 

Gradient checkpointing reduces memory consumption during backpropagation by storing only a subset of intermediate activations from the forward pass and recomputing the missing ones when needed. 
In standard backpropagation, all activations must be kept in memory because each layer’s gradient depends on its forward output; this creates a memory footprint that grows linearly with depth and often becomes the bottleneck in training large models. 
Gradient checkpointing breaks the computation graph into segments, saves only the boundary activations, and discards the rest. 
During the backward pass, the discarded activations are recomputed on the fly by running partial forward passes within each segment. 
This trades additional compute for a substantial reduction in memory usage. 
The memory cost can drop from linear to roughly the square root of the number of layers in optimal schemes, enabling deeper networks or larger batch sizes on the same hardware. 
The extra compute cost is typically a factor of 1.5–2×, but the trade-off is advantageous when memory is the limiting resource. Modern implementations generalize this idea through customizable checkpoint policies, selective recomputation, and integration with automatic differentiation frameworks, making gradient checkpointing a standard technique in large-scale training, especially for transformer architectures and diffusion models where activation tensors dominate memory load.

ODE solvers compute trajectories by applying numerical integration steps, with cost driven by repeated evaluations of the vector field. Higher-order or adaptive methods improve accuracy but increase per-step compute, so total complexity depends on both the solver order and the number of required steps. In neural ODEs, most of the expense comes from evaluating the neural network that defines the dynamics.
Using the autograd engine to compute gradients through a ODE solution keeps the entire computational graph in memory, resulting in $\mathcal{O}(NL)$ scaling with a second pass over the graph in reverse direction required for the backward pass with a total cost of $\mathcal{O}(NL)$.

The adjoint method recomputes the state trajectoy and backpropagates the gradients by integrating a separate reverse-time ODE, avoiding storage of intermediate states and keeping memory use low. 
Compared to gradient checkpointing it does not require storing any values besides the vectorfield as it recomputes both the trajectory and the gradients during the backward pass.
This leads to a memory complexity of $\mathcal{O}(L)$ as only the network parameters and their activations need to kept in memory at any time.
This comes at the price of additional solver calls and potential numerical instability or even divergence from the trajetory of the forward pass, since the backward integration can amplify errors and often requires tighter tolerances.
After $\mathcal{O}(NL)$ steps of compute in the forward pass, the adjoint method recomputes the forward pass to reconstruct the trajectory as well as one backward pass to compute the gradient, resulting in an additional $\mathcal{O}(2NL)$ compute complexity during the backward pass. 
As a result, adjoint-based gradients are memory-efficient $\mathcal{O}(L)$ but computationally more expensive with $\mathcal{O}(3NL)$.
Solving a reverse time ODE comes with potential numerical issues as the recomputed trajectories diverge and subsequently lead to diverging gradients.

Checkpointing provides a middle ground by storing selected solver states and recomputing segments during backpropagation. 
This reduces memory usage to $S = N/K$ at any given time where $S$ is the segment length.
Storing $K$ checkpoints in the form of a state of the ODE is negligible as it amounts to $K$ samples.
For $K$ checkpoints, we only have to solve $S$ steps at any one time, while the recomputation between the $K$ checkpoints amounts to going forward and backward on the segments compute complexity of $\mathcal{O}(2K S L) = \mathcal{O}(2NL)$.
It balances compute and memory more predictably, making it a practical choice for many neural ODE applications.

To quantitatively measure the performance of both the gradient checkpointing and adjoint method, we evaluated each gradient against the ground truth gradient obtained with numerical differentiation.
Figure~\ref{fig:autogradvsadjoint} shows the gradient error over increasing modelling dimensions as a funtion of the step size and the checkpointing interval. 
For the Euclidean part, we integrated the function $dx_t = \cos(t/2\pi) \cdot x_t$ and for the $SO(3)$ part, we integrated $dR_t = \cos(t/2\pi) \sum_a^3 v_a T^a \theta(v)$ with $v = [1.0, -0.2, 0.1]^T$ and $\theta(v) = \|v\|$. The initial condition was sampled from a Gaussian distribution in Euclidean space and the respective axis-angle space for the $SO(3)$ elements.
The dimensionality was determined by the number of $SE(3)$ ODE's integrated in parallel for $t\in[0,1]$.
We used this toy ODE's to establish the numerical behavior of the autograd checkpointing and the adjoint methods for backpropagating gradients through ODE solvers.

The analysis suggests that the adjoint method with a checkpointing at every step yields gradients very close to its autograd gradient.
Decreasing the checkpointing frequency requires recomputing longer parts of the trajectory which accumulates errors.
Overall a trend emerges that decreasing the step size is beneficial for both adjoint and autograd gradients, and particularly for very small stepsizes, the difference reduces significantely.

\begin{figure}[h!]
  \centering
\includegraphics[width=0.8\linewidth]{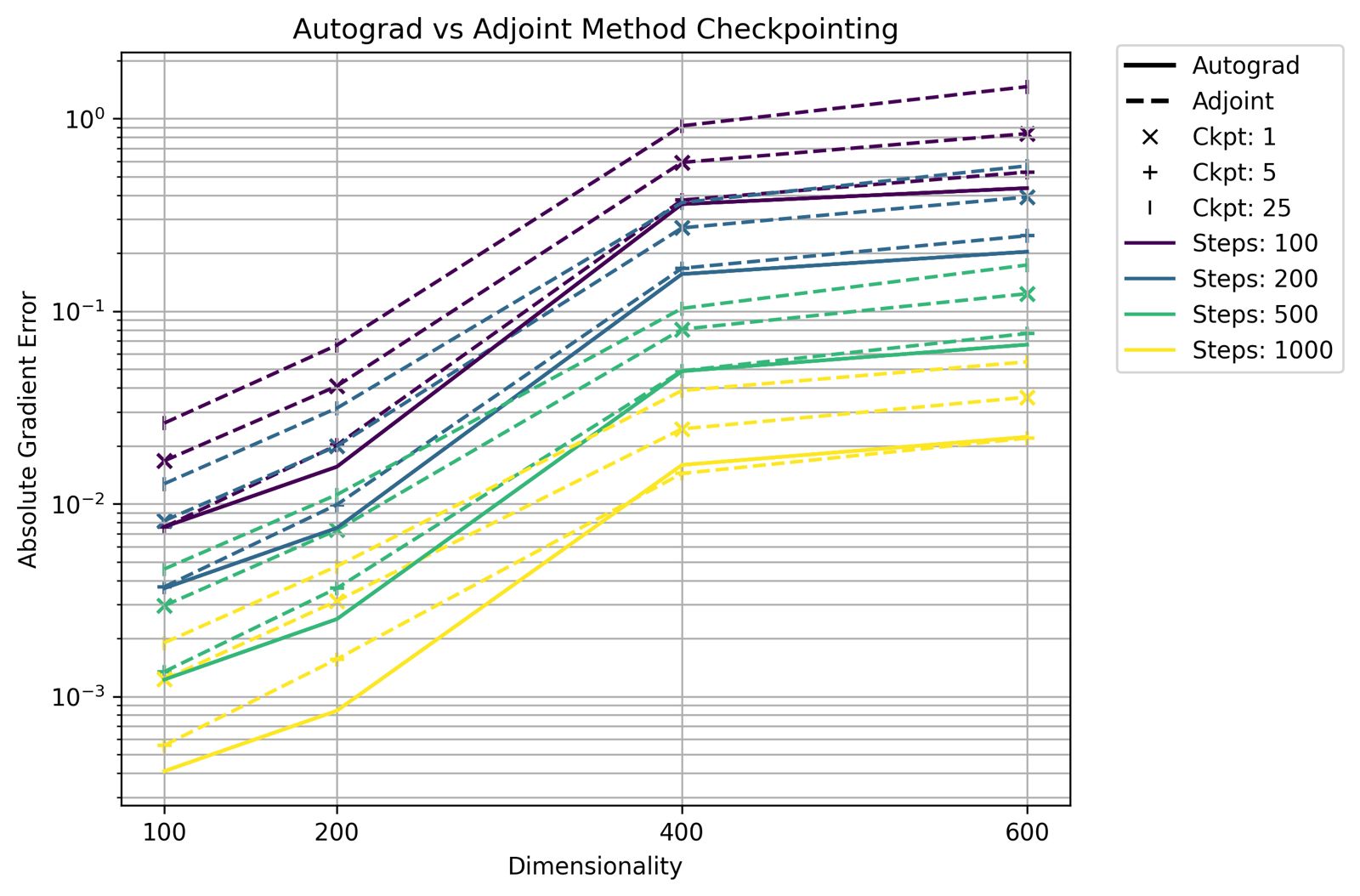}
  \caption{Gradient Error as a function of dimensionality and the number of steps}
  \label{fig:autogradvsadjoint}
\end{figure}

We also extended this analysis to the FrameFlow model \cite{frameflow} model but due to the computational cost of numerical differentiation for the ground truth gradient refrained from evaluating multiple checkpointing frequencies.
For this reason we kept the checkpointing frequency fixed at every five steps.
One can observe that the adjoint gradients backpropagated through the $SE(3)$ ODE generally more imprecise than their autograd counterparts.
For commonly used stepsizes of $1/25$ to $1/100$, the total deviation can be quite significant, but correcting for the per dimension gradient deviation, the gradient error is small.

\begin{figure}[h!]
  \centering
\includegraphics[width=0.8\linewidth]{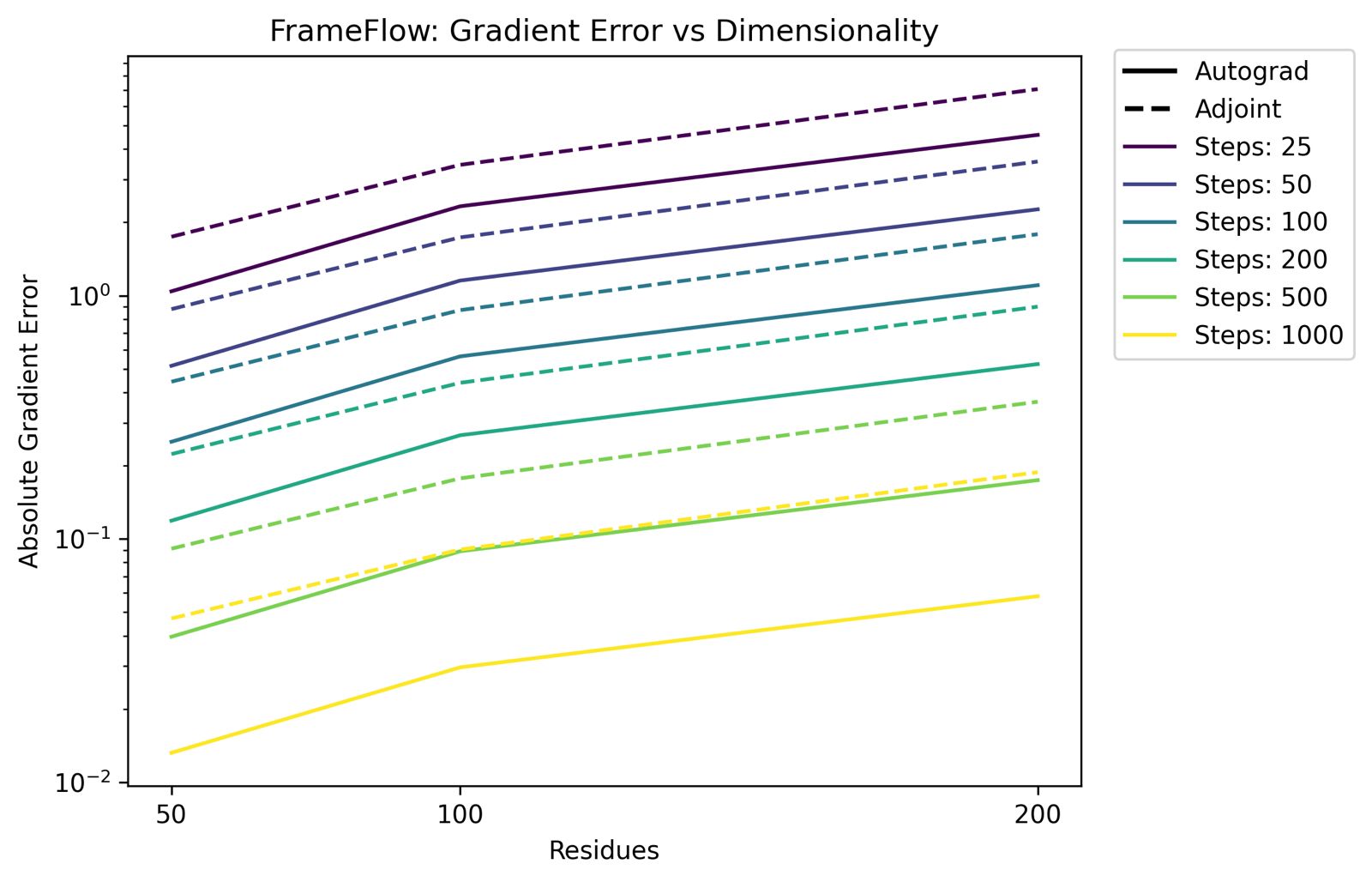}
  \caption{Gradient Error as a function of dimensionality and the number of steps}
  \label{fig:autogradvsadjoint}
\end{figure}

Finally, we compare the gradient error of the AlphaFlow architecture \cite{alphaflow}.
Again, adjoint gradient estimators perform worse than their autograd counterparts.
It should be noted that AlphaFlow is a flow defined on the $C_\beta$ carbon atoms and does not model a rotation group, compared to the previous two models.
We want to note that the evaluation of these gradients were on the order of hours, particularly for the highest step sizes.
\begin{figure}[h!]
  \centering
\includegraphics[width=0.8\linewidth]{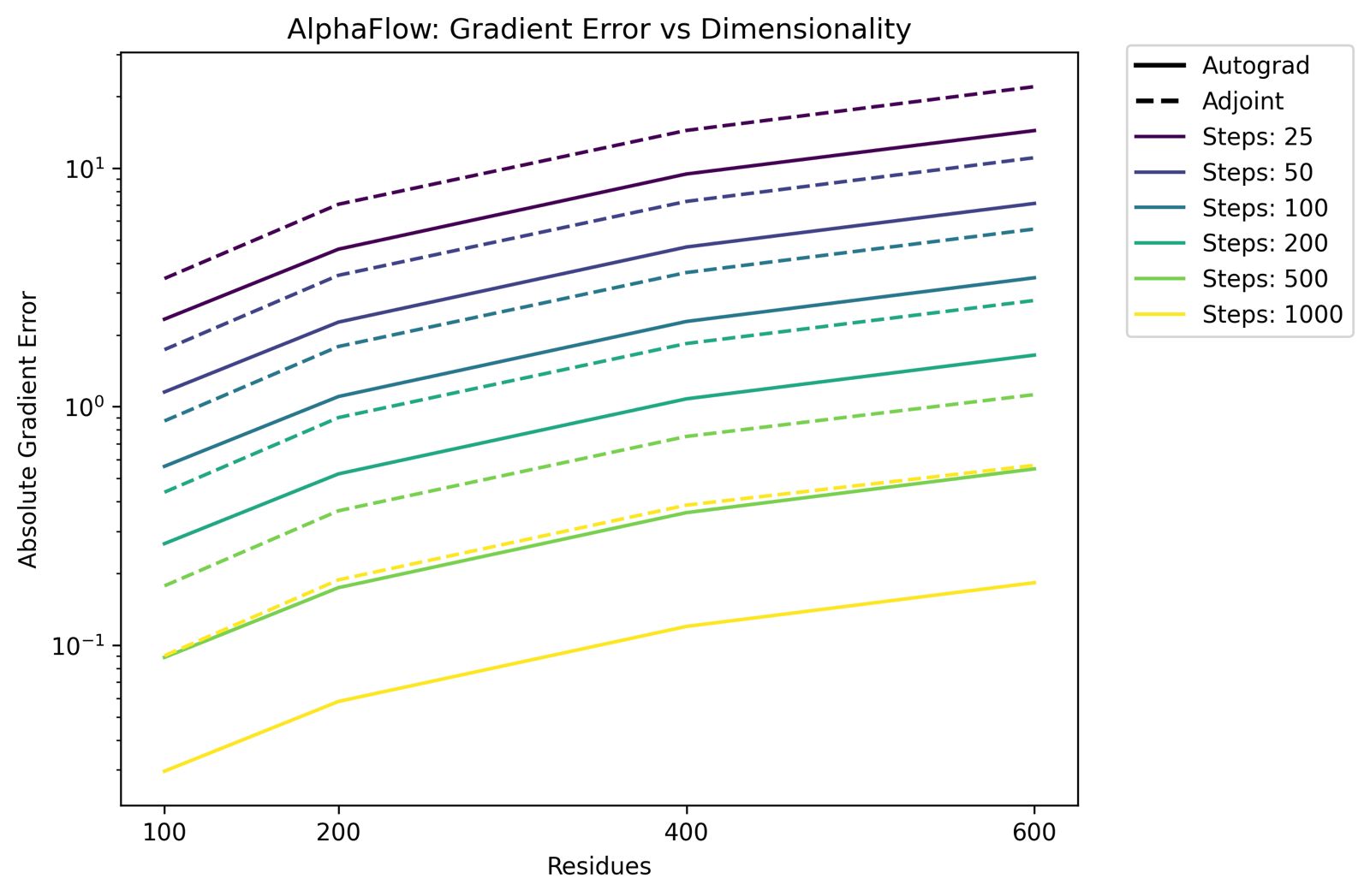}
  \caption{Gradient Error as a function of dimensionality and the number of steps}
  \label{fig:autogradvsadjoint}
\end{figure}

\section{Experiments}
\label{app:experiments}

\subsection{SO(3) Manifolds}

We generate a manifold in the space of $SO(3)$ matrices.
To plot the rotation matrices easily, we parameterize the SO(3) elements in their axis-angle vectors. The manifold in the axis-angle representation was generated with the S-curve function of sklearn, \verb+sklearn.datasets.make_s_curve(n_samples=250000, noise=0.1)+.

The flow matching model was parameterized with a four layer MLP with SELU activation functions and 64 hidden units in each of the hidden layers.
For training, the model was optimized with Adam using a learning rate of $1e^{-3}$ over 20 epochs on the training data with conditional flow matching to enforce that the predicted flow on the manifold of SO(3) matrices closely matches the target logarithmic map on the manifold.
During sampling, the vector field was integrated with the Crouch-Groussmann integrator with a step size of $dt=1/200$ from $0$ to $1$.

For the sample optimization, a target sample $z^*$ was experimentally determined and a second random sample was drawn as the initial value $\smash{x^{(0)} = g(z^{(0)})}$ of the optimization problem.
The mean squared error (MSE) between the target sample $x^*$ and the sample $x^{(0)}$ being optimized served as the loss function $\mathcal{L}$.
The sample $z^{(i)}$ was then optimized by calculating the gradient $d\mathcal{L}(g(z^{(i))}) / dz^{(i)}$ and performing 20 steps of gradient descent in the $SO(3)$ base space with a base space learning rate of $0.1$. 

The tangent space was calculated by evaluating the Jacobian matrix $J_g$ with \verb+torch.nn.functional.jacobian+ at point $z$ and its corresponding target space point $x = g(z)$.
Subsequently, the two largest eigenvalues with the corresponding left columns were computed with a SVD and plotted with a scaling of 1.5 the point $x$.

\subsection{Secondary Structure Optimization}
\label{app:experiments_secondary}

The secondary structure experiments were conducted on top of the FrameFlow codebase \cite{frameflow}.
For training and architectural details, we refer to the manuscript.

We used the model parameters inn the provided model checkpoints by the original authors.
The samples $x$ with the time index $t \in [0,1]$ consist of backbone elements of each residue centered on the Ca carbon atom.
Each of the $L$ backbone elements is parameterized by its translation and rotation and are elements of the SE(3) group.
The flow generates an element $(t, R)$ of the $SE(3)$ group for each residue.
Upon completion of the sampling process, these transformations are applied to the idealized backbone coordinates.

For the experiment described in the main part of the paper, the loss function $\mathcal{L}$ was calculated as a function of the distance between the atom positions in the first and the last residue.
The distance loss was implemented by taking the ReLU function over the difference between the desirable target distance and the actual distance, i.e. 
$\mathcal{L} = \left(\max (0, d^* - d(x^{(i)}_{0}, x^{(i)}_{L}) )\right)^2$ where the target distance $d^*$ between the atom positions in the two termini $x^{(i)}_0$ and $x^{(i)}_L$ in the backbone structure of length $L$ at optimization step $(i)$ of the structure was chosen as $d^*=L/2$.
The distance metric was chosen as the mean squared error $d(x, y) = \sqrt{\sum_j (x_j - y_j)^2}$ over the atoms $j$ in two residues $x$ and $y$.
The order of subtraction $\max(0,d^* - d)$ ensures a loss that becomes zero if $d \geq d^*$ and that the two termini are repelled from each other.
We also experimented with an attraction loss in the sense of $\max(0, d - d^*)$ which worked equaly well.
In fact, we were able to dynamically switch losses during optimization and the termini were pulled apart and together smoothly as determined by the loss with all intermediate optimization states $x^{(i)}$ staying on the data manifold.
When switching between attracting and repelling loss function it was beneficial to reset the momenta of the adaptive optimizers like Adam to accelerate the optimization.
For sampling we integrated with the Euler integration method with a step size of $dt = 1/25$.
The learning rate for the optimization was chosen at a constant $0.5$ for $300$ optimization steps.

For the optimization of the dihedral angles according to the ABEGO scheme, we utilized a differentiable implementation of Rosetta energy function in PyTorch, i.e. tmol \cite{tmol}.
We use a specific component of the full energy function that takes the dihedral angles of neighboring backbone elements as an input. We modified this component such that it increases the contribution for dihedral angles corresponding to $\beta$-sheets. 

In more detail, the combined Ramachandran/$\textrm{p\_aa\_pp}$ energy term in tmol interpolates the energies (the negative-log probabilities) stored in a toroidal-grid lookup-table over phi and psi dihedrals. To push the generated samples away from the alpha-helical region of conformation space, we created our own custom lookup table where we dramatically increased the energies for an elliptical region of the Ramachandran map that covers the alpha-helical region. The energies at the edge of this ellipse are not uniform, so our first crude assignment of energies inside this ellipse that sloped away from a central peak to a particular value at its edge frustrated the minimizer, trapping conformations in the weird minima at the edge of the ellipse. In our second attempt, we instead labeled each grid cell inside the ellipse by the closest cell outside of the ellipse and then interpolated the energies along the distance from the center of the ellipse (to which we assigned an energy of +20) to the ellipse's edge. This led to the desired behaviour and we therefore chose this approach. The resulting grid is shown in Figure~\ref{fig:rama_guidance_pot}.

\begin{figure}[htbp]
  \centering
\includegraphics[width=0.5\linewidth]{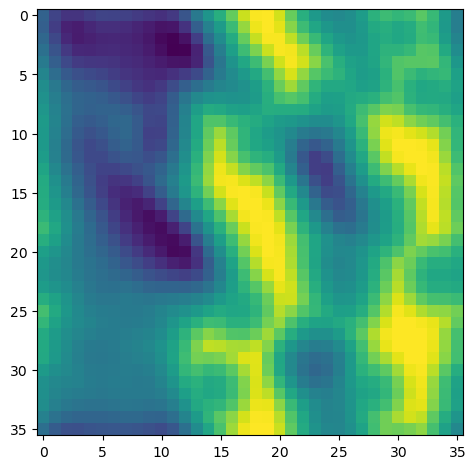}
  \caption{Grid used for Ramachandran guidance.}
  \label{fig:rama_guidance_pot}
\end{figure}

\begin{figure}[htbp]
  \centering
\includegraphics[width=0.9\linewidth]{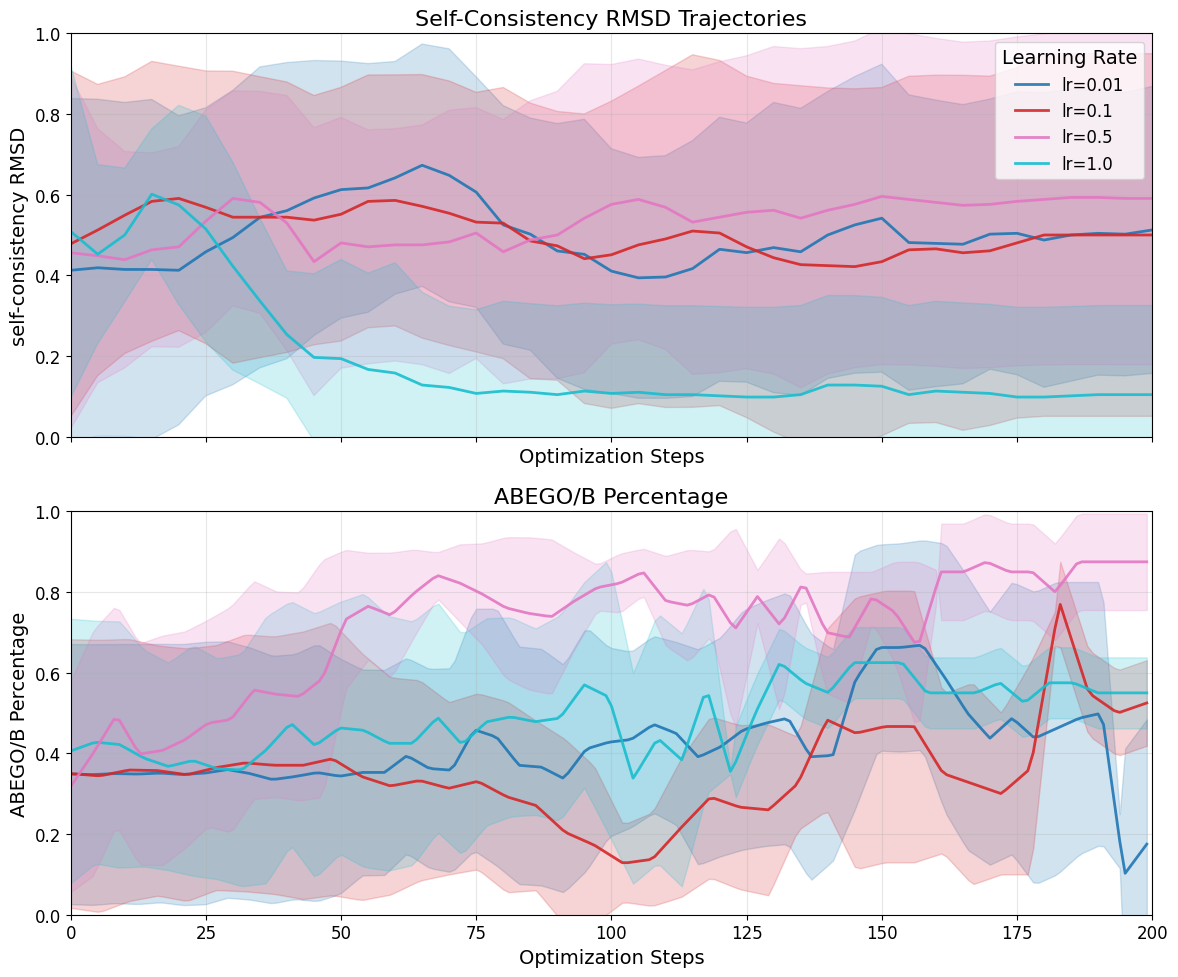}
  \caption{Self-consistent RMSD and the optimized ABEGO percentage over the course of diffeomorphic optimization.}
  \label{fig:scRMSDvsABEGO}
\end{figure}

Three important metrics in machine learning aided protein design are \emph{designability}, \emph{diversity} and \emph{novelty}.
We consider designability the most pertinent to our task as diversity and novelty are of lesser interest when trying to converge on a single optimal structure.
The optimization towards a single structure makes diversity collapse as expected. 
Similarly, since our aim is not to generate diverse samples but rather optimize towards a singular structure, novelty is a misplaced metric in this case.
Empirically, the designability of the optimization trajectories agree well with the designability of the initial values drawn from the generative model.
Only for larger learning rates of $1.0$ does the designability experience a significant drop while at the same time only marginally improving the percentage of dihedral angles classified as $B$ in the ABEGO scheme.

Since the FrameFlow model does not predict side chains, we interpreted all residues as alanine.
We applied the sampled translations and rotations to the idealized backbone atom14 positions to obtain the sampled backbone geometry.
The score function was minimized with gradient descent in the base space with a learning rate of $0.1$ for $500$ steps.

We implemented several versions of guidance to obtain maximally competitive baseline results.
Specifically, we implemented both a single step and and a multi step denoiser combined with the loss function described above. The guidance was started at $t=0.1$. We selected the starting point of the guidance signal by hyperparameter line search over $t=0.8, 0.1, 0.25, 0.5, 0.75$. We also found that it was benefical to stop sampling early at $t=0.98$ to prevent distortive numerical pathologies that arose due to the divisor of the optimal transport transport path close to $t=1$. The guidance vector field was obtained by evaluating the loss on the atom14 positions of each backbone residue and computing the gradient with respect to the corresponding element in the base space. For weighting the guidance term, we tested the theoretically motivated weighting function $b_t$ from \cite{lipman2024flow} but obtained better results with a constant, time-independent weight of $w=25$. This value was selected using a line-search hyperparameter sweep for $w=0.1, 1, 10, 25, 50, 100, 1000$.
For evaluating the guidance term on a sample $x_t$, we tested both taking the direct gradient of the loss function, $\partial{L}(x) |_{x = g(x_t)} / \partial x$, and backpropagating back through the denoiser $\partial{L}(g(x_t)) /  \partial x_t$. We chose the latter as it yielded better results. Experimentally, we also found that only using the translation term (and not the rotational part) as guidance improved the performance. This may be due to the fact that on $SO(3)$, there is no equivalence between score and denoising as in Euclidean space. We furthermore carefuly ablated the number of denoiser steps used in multistep denoising. This is because fair comparison to diffeomorphic optimization should allow for a generous budget of denoising steps as the method is numerically costly. As shown in Table~\ref{tab:denoising_abl}, even significant increase in the denoising steps does not lead to a pronounced gain in performance.

To demonstrate that the numerical cost of diffeomorphic gradient calculations can be reduced by coarse-grained backpropagation through the solver, we performed the following ablation study: we measured the performance of diffeomorphic optimization as a function of the number of steps used for sampling and calculating gradients during optimization.
To this end, we ran diffeomorphic optimization with $10, 25, 50, 75$ and $100$ integration steps during optimization, and evaluated the optimized sample after optimization with $1000$ integration steps.
The results are shown in Figure~\ref{fig:diffeo_int_steps}.
Compared to guidance, even a very small number of steps used for sampling during optimization yields good performance. For this experiment, we measured the performance over 10 seeds as compared to 50 seeds in the main part.

\begin{figure}[htbp]
  \centering
\includegraphics[width=0.75\linewidth]{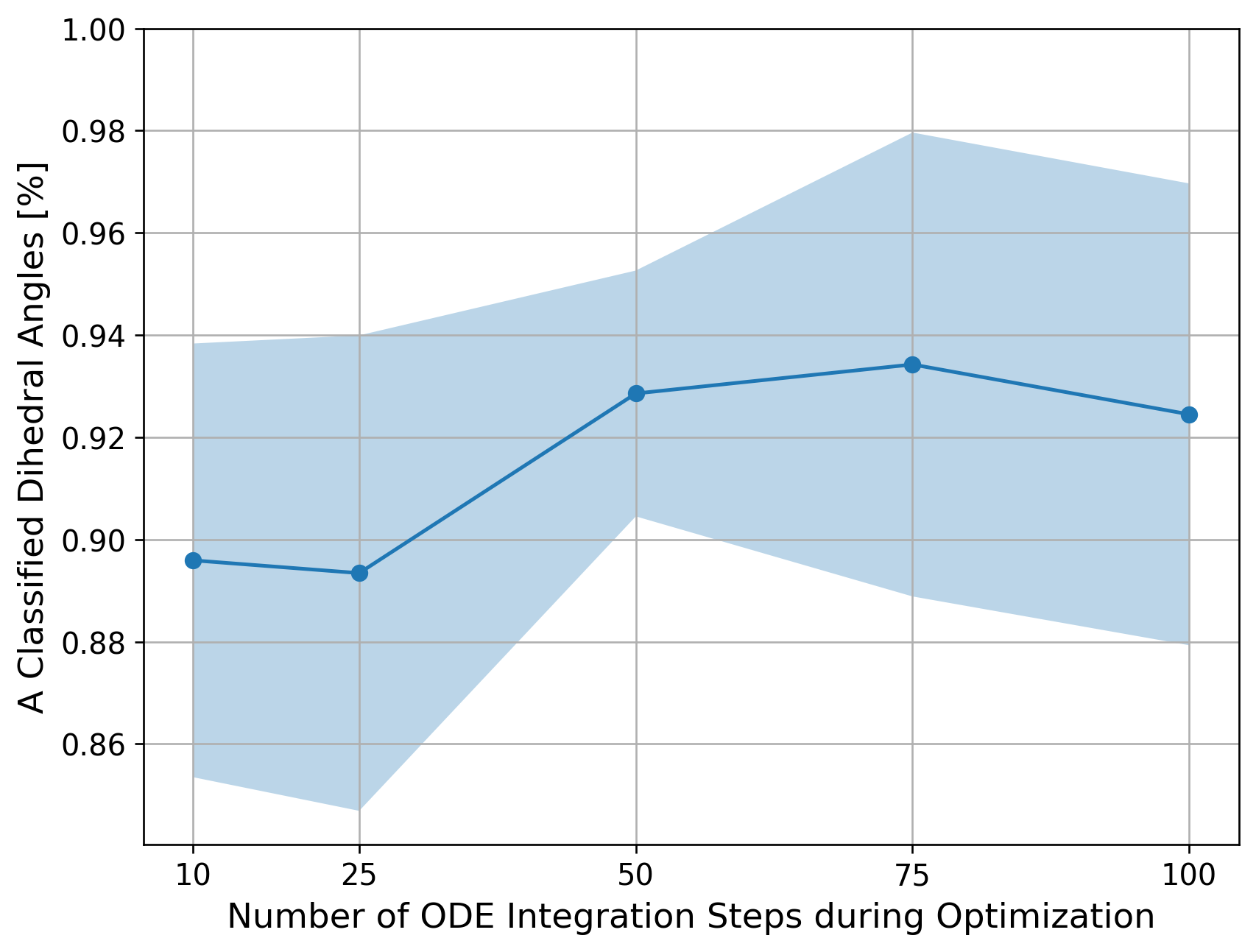}
  \caption{The performance of diffeomorphic optimization (higher is better) as a function of the number of ODE integration steps taken for each evaluation during optimization. After optimization, we evaluated all optimized samples with a $1000$ steps. Diffeomorphic optimization is robust even when using just 10 integration steps during optimization.}
  \label{fig:diffeo_int_steps}
\end{figure}

\begin{table}[h!]
\centering
\begin{tabular}{
  c
  c
  c
  S[table-format=2.1(2.1)]
  S[table-format=2.1(2.1)]
}
& & & \multicolumn{2}{c}{Secondary Structure} \\
& & & {$A$} & {$B$} \\
\midrule
Unguided & & & 53.5(23.3) & 32.8(20.1) \\
\midrule 
& ODE Steps & Denoiser Steps & & \\
\cmidrule(lr){2-5}
Guidance & 200 & 1 & 63.3(25.1) & 27.7(23.3) \\
&  & 50 & 63.5(24.8) & 27.4(23.1) \\
&  & 100 & 63.8(24.5) & 27.1(22.9) \\
&  & 200 & 64.1(24.1) & 26.8(22.6) \\
\cmidrule(lr){2-5}
& 500 & 1 & 64.2(24.4) & 26.6(22.4) \\
&  & 50 & 64.4(24.0) & 26.3(22.1) \\
&  & 100 & 64.7(23.7) & 26.0(21.9) \\
&  & 200 & 65.0(23.4) & 25.8(21.6) \\
\cmidrule(lr){2-5}
& 1000 & 1 & 65.0(23.6) & 25.5(21.5) \\
&  & 50 & 65.1(23.4) & 25.3(21.3) \\
&  & 100 & 65.2(23.2) & 25.1(21.0) \\
&  & 200 & 65.3(23.1) & 24.9(20.8) \\
\midrule
& ODE Steps & Optim Steps & & \\
\cmidrule(lr){2-5}
Diffeomorphic & 25 & 500 & \bfseries 91.3(8.4) & \bfseries 7.4(7.3) \\
\end{tabular}
\vspace{10pt}
\caption{The addition of more compute for guidance leads to diminishing returns. Increasing the the number of integration steps both in the flow sampling as well as in the denoiser does not automatically lead to a higher target metric as measured with the percentage of $A$ classified dihedral angles. All guidance terms were weighted with a constant $w=25$. Diffeomorphic optimization samples with $25$ integration steps and applies $500$ gradient descent steps in the base space. Diffeomorphic optimization yields higher $A$ classification scores with a substantially lower standard deviation in performance.}
\label{tab:denoising_abl}
\end{table}

\subsection{Docking Optimization}
\label{app:experiments_docking}


For docking optimization we used the DiffDock code base.
This diffusion model samples from the manifold of permissible rotational degrees of freedom of a small molecule on $SO(2)$ and its global translation and rotation on $SE(3)$.
In order to make the sampling process differentiable, we implemented the probability flow formulation of the diffusion models which provided a deterministic map between base space and target space.
For a differentiable loss, we considered two different metrics: the confidence head of the diffusion model and the OpenDock scoring function \cite{opendock}.

Opendock is a physics-inspired scoring function and implements the Vina Score in a differentiable manner in PyTorch. One of the technical challenges of incorporating OpenDock as a suitable loss function were posed by its reliance on PDBQT file formats.
In addition to the positional information stored in the PDB file format, the PDBQT file format additionally stores the partial charges, torsional bonds, and autodock atom types.
To work around this constraint, we first saved the ligand conformations generated by DiffDock to disk in the SDF file format, converted it to the PDBQT and reloaded it with the additionally generated partial charges created during the conversion.
At this point the computational graph connecting the base space variables to inputs of the differentiable loss scoring function was suspended due to the reading and writing to disk.
To reconnect the computational graph with the input of the Opendock scoring function, we substituted the atom positions in the loaded PDBQT data structure with the atom positions originating from the DiffDock model. 

The confidence head of the diffusion model is a neural network and was trained to predict whether the generated conformation is within $2$ Angstroms of the ground truth.
We additionally include on the scoring of OpenDock into our experimental setup which is a differentiable implementation of Vina.
Following the original authors recommendation on the confidence prediction, we screened for the entries in the publicly available PDBBind test data set that were assigned a confidence logit of larger than 0 by the DiffDock confidence head.
This resulted on a subset of roughly 78 PDBBind ligand-protein pairs on which the results were reported.
We perform 10 optimization steps with an Adam-style optimizer with a learning rate of $0.1$.
Experimentally, we observe that the gradients of the rotation and torsional degrees of freedom play a dominant role in the norm of the total gradient.
To counteract this, we introduced separate scalings for the learning rates of the rotation and torsions and multiply the learning rates for these two sets of degrees of freedom by an additional $0.1$.

To ensure fair comparison, we draw more samples for the sampling baseline in order to match the computational budget used by diffeomorphic optimization.
Since both checkpointing and the adjoint method recompute the activations during the backward pass through the ODE solver, we used a budget corrector of 3$\times$ a single sampling step.
This budget corrector number could be reduced more with coarser checkpointing and fewer steps, but we observed favorable experimental results even for this setting and thus did not explore this direction further.

\subsection{Tertiary Structure Optimization}
\label{app:experiments_tertiary}

We built on the alphaflow codebase\cite{alphaflow} and chose the ESMFold-based architecture as it does not require costly MSA. We ensured that diffeomorphic optimization works both for distilled and non-distilled case but large-scale testset-wide experiments were conducted with the latter as it leads to significant speed-up.
The model ingests the pairwise distances between random atom positions to predict the backbone and the side chain atom positions conditioned on the sequence information of the protein.
As a differentiable loss function, we used the tmol implementation of of $\textrm{beta\_nov2016\_cart}$ Rosetta score function which support pytorch autograd to compute gradients.
As the architecture contains non-differentiable distograms, we implemented a differentiable straight-through version thereof harnessing the sigmoid function $\sigma$,
\begin{align*}
    d_{soft} &= \sigma(\beta \cdot (d - lower)) \cdot (1- \sigma(\beta \cdot (d - upper)) \\
    d_{gram} &= \text{detach\_grad}(d) + (d_{soft} - \text{detach\_grad}(d_{soft}))
\end{align*}
where $lower$ and $upper$ denote the lower and upper bound of the discrete bin in which the distance $d$ should have been placed.
While the forward pass is unaffected by the soft distogram implementation, the gradients are affected by the smooth approxiation.
Higher values of the $\beta$ temperature result in a steeper sigmoid function at the cost of a stiffer gradient surface.

As a baseline we used the state-of-the-art Rosetta Relax protocol \cite{rosetta} without any diffeomorphic optimization.
For random sidechain repacking, Rosetta Relax uses $n$ seeds which are commonly referred to as structures in the Rosetta community. This is followed by gradient descent in which the repulsive terms in the energy function are adiabatically increased with a linear schedule. Packing and gradient-based minimization is then repeated $k$ times. 

In Figure~\ref{fig:rosettacompute} we evaluated the Rosetta Relax baseline with an increasing amount of compute measured as the product of the number of seeds $1, 3, 5$ with the number relaxation cycles $3,5$. We were not able to discern an improvement compared to diffeomorphic optimization in conjunction with Rosetta Relax.

For diffeomorphic optimization, we used Adam with a learning rate of $0.1$ and otherwise standard pytorch hyperparameters. This was then followed by standard Rosetta relax to reduce numerical cost. The performance of this protocol as a function of diffeomorphic optimization steps is shown in Figure~\ref{fig:rosetta2}.

In addition to the experiments reported in the main part of the paper, we also compared the performance of our diffeomorphic Rosetta Relax to a sampling-based baseline. For this baseline, we simply sample from AlphaFlow with the equivalent computational budget used by diffeomorphic optimization. We select the sample with the lowest energy which is then minimized by the same standard Rosetta relax as used in the last step of the diffeomorphic protocol. As shown in Figure~\ref{fig:rosetta_otherbaseline}, diffeomorphic optimization again outperforms this baseline.

\begin{figure}[htbp]
  \centering
\includegraphics[width=0.8\linewidth]{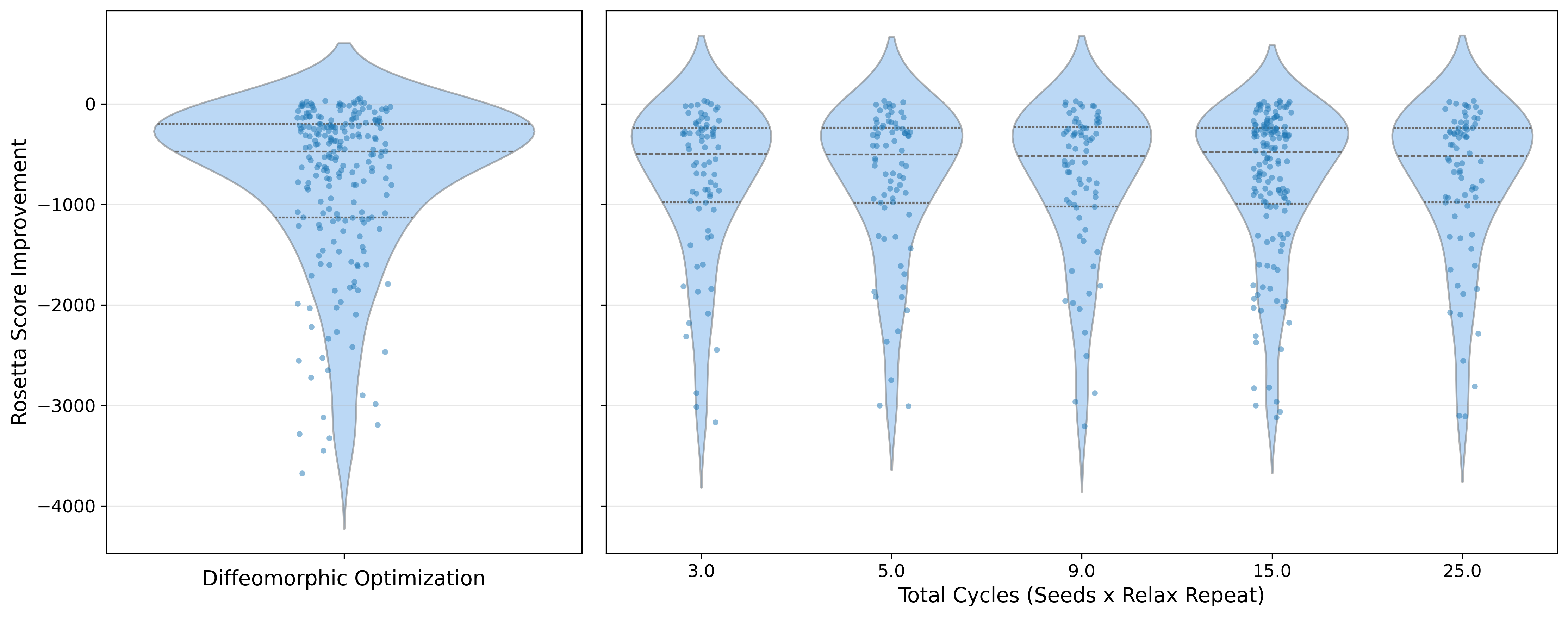}
  \caption{Diffeomorphic optimization of the Rosetta energy function: each point denotes a prediction of ESMFlow for an element of the pdb test set of Alphaflow for which we measure the improvement $E_{\textrm{diffeo}}-E_{\textrm{baseline}}$ of diffeomorphic Rosetta Relax over standard Rosetta relax. \textbf{Left:} Diffeomorphic Rosetta Relax outperforms Rosetta Relax and improves the Rosetta score with the default Rosetta Relax values of three structures and three repeated relaxations. \textbf{Right:} This improvement holds across an increasing computational budget as measured in terms of the total number of relaxations.}
  \label{fig:rosetta}
\end{figure}


\begin{figure}[!htbp]
  \centering
\includegraphics[width=\linewidth]{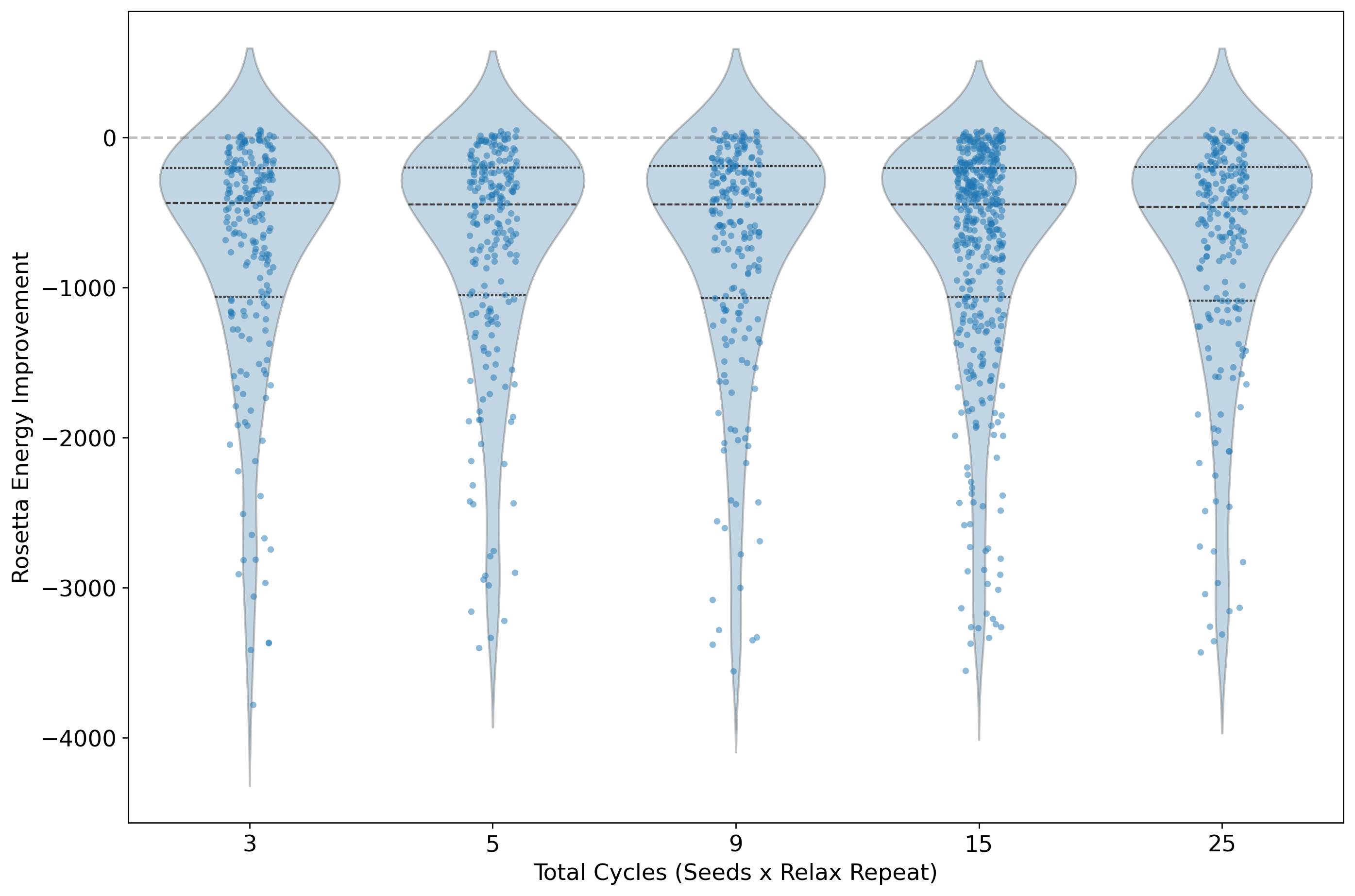}
  \caption{Diffeomorphic optimization in combination with Rosetta Relax yields improved (lower) Rosetta energies when directly compared to Rosetta Relax. This statement holds even with increasing computational budget for the baseline. In this experiment, we varied the number of Rosetta Relax random seeds (referred to as structures in the Rosetta community) and the number of relaxations for each random seed. The product of both Rosetta Relax random seeds and relaxations is plotted on the horizontal axis. We can observe no substantial benefits of providing Rosetta Relax with additional compute budget compared to diffeomorphic optimization which consistently improves the final Rosetta energy.}
  \label{fig:rosettacompute}
\end{figure}

\begin{figure}[!htbp]
  \centering
\includegraphics[width=\linewidth]{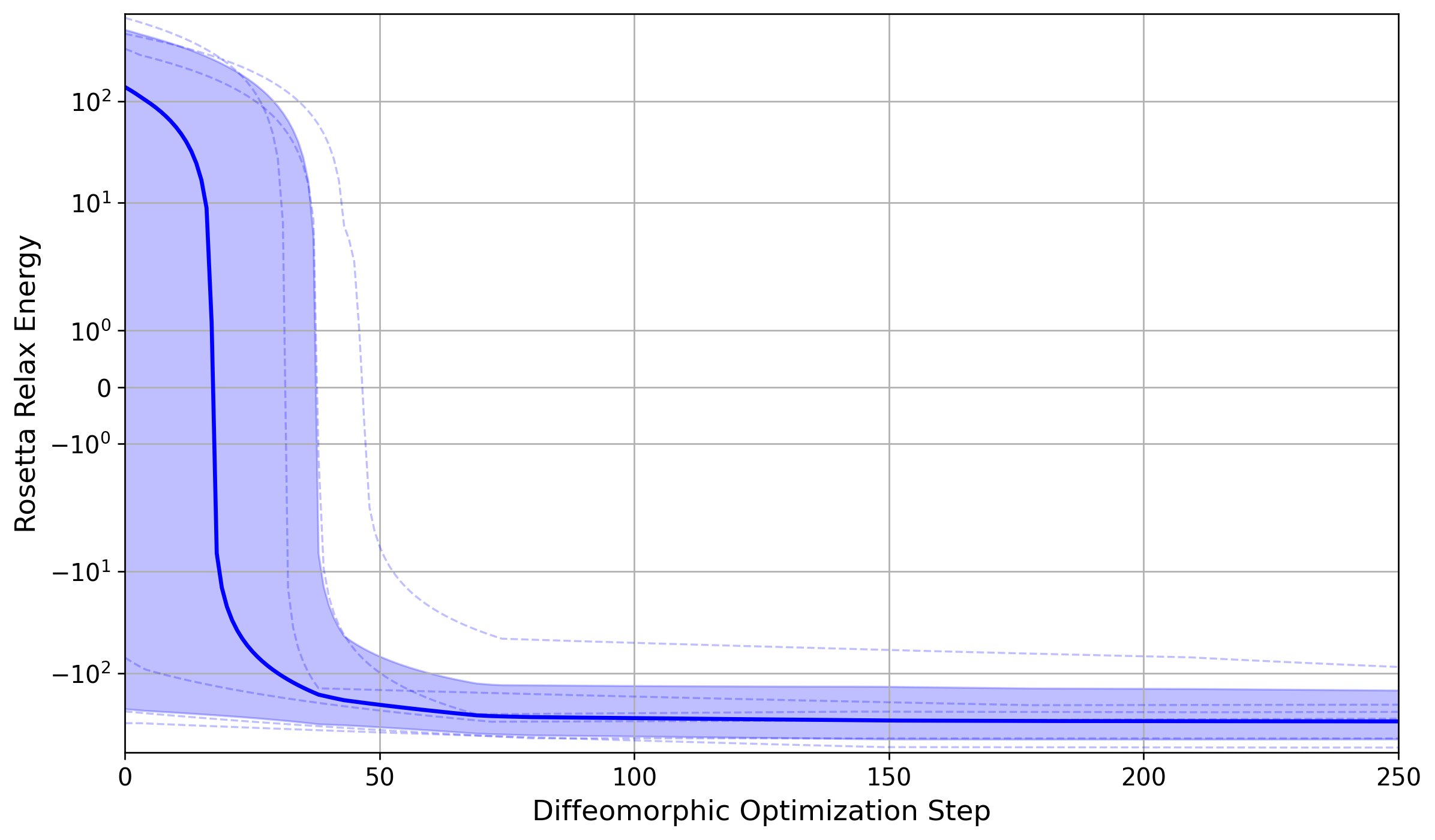}
  \caption{Diffeomorphic Optimization in combination with Rosetta Relax can achieve substantial energy reductions in a small number of steps. This is exemplified by the high energies at the beginning of the optimization curves which corresponds to running Rosetta Relax without a diffeomorphic optimization step which are subsequently is minimized by several orders of magnitudes while remaining stable.}
  \label{fig:rosetta2}
\end{figure}

\newpage

\end{document}